\documentclass{article}

\usepackage[utf8]{inputenc}
\usepackage{bbm}
\usepackage{mathtools}
\usepackage{amsthm}
\usepackage{comment}
\usepackage{appendix}

\usepackage{amsfonts}
\usepackage{bm}
\usepackage{url}
\usepackage[margin=1in]{geometry}

\usepackage{natbib}
\usepackage{microtype}
\usepackage{bbm}
\usepackage{mathtools}
\usepackage{amsthm}

\usepackage{appendix}

\usepackage{amsfonts}
\usepackage{bm}
\usepackage{url}
\usepackage{graphicx} 
\usepackage{enumitem}
\newtheorem{theorem}{Theorem}
\newtheorem{example}{Example}

\newtheorem{lemma}{Lemma}

\newtheorem{corollary}{Corollary}
\newtheorem{defn}{Definition}
\newtheorem{prop}{Proposition}

\newcommand{\E}{\mathbb{E}}
\newcommand{\N}{\mathbb{N}}
\newcommand{\R}{\mathbb{R}}

\newcommand{\argmin}{\text{argmin}}
\newcommand{\conv}{\text{conv}}

\newcommand{\one}{\mathbf{1}}

\newcommand{\classifier}{\phi}

\newcommand{\X}{\mathcal{X}}
\newcommand{\Y}{\mathcal{Y}}
\newcommand{\Z}{\mathcal{Z}}
\newcommand{\F}{\mathcal{F}}
\newcommand{\G}{\mathcal{G}}

\newcommand{\T}{\mathbb{T}}

\newcommand{\sample}{\mathcal{D}}

\newcommand{\lossFunction}{\mathcal{L}}

\newcommand{\HClass}{\mathcal{H}}
\newcommand{\measurableMaps}{\mathcal{M}}
\newcommand{\rademacherComplexity}{\mathfrak{R}}

\newcommand{\hypothesisClassBound}{\beta}
\newcommand{\probDistribution}{P}

\newcommand{\sampleXY}{\sample_{X,Y}}

\newcommand{\numClasses}{q}

\newcommand{\xSequenceSizeN}{{\bm{x}}}
\newcommand{\xySequenceSizeN}{{\bm{z}}}
\newcommand{\zSequenceSizeN}{{\bm{z}}}
\newcommand{\wSequenceSizeNQ}{{\bm{w}}}
\newcommand{\sigmaSequenceSizeN}{{\bm{\sigma}}}

\newcommand{\lossFunctionBound}{B}
\newcommand{\riskArg}{\mathcal{E}}
\newcommand{\risk}{\riskArg_{\lossFunction}}
\newcommand{\empiricalRiskArg}{\hat{\mathcal{E}}}
\newcommand{\empiricalRisk}{\empiricalRiskArg_{\lossFunction}}

\newcommand{\actionSpace}{\mathcal{V}}

\newcommand{\selfLipschitzConstant}{\lambda}
\newcommand{\selfLipschitzExponent}{\theta}

\newcommand{\metricSpace}{\mathcal{M}}
\newcommand{\dist}{\rho}
\newcommand{\coveringNumber}{\mathcal{N}}

\renewcommand{\sampleXY}{\sample}

\renewcommand{\lossFunctionBound}{b}

\title{Optimistic bounds for multi-output prediction}
\date{}
\author{Henry WJ Reeve and Ata Kab\'{a}n}

\begin{document}

\maketitle

\begin{abstract}
We investigate the challenge of multi-output learning, where the goal is to learn a vector-valued function based on a supervised data set. This includes a range of important problems in Machine Learning including multi-target regression, multi-class classification and multi-label classification. We begin our analysis by introducing the self-bounding Lipschitz condition for multi-output loss functions, which interpolates continuously between a classical Lipschitz condition and a multi-dimensional analogue of a smoothness condition. We then show that the self-bounding Lipschitz condition gives rise to optimistic bounds for multi-output learning, which are minimax optimal up to logarithmic factors. The proof exploits local Rademacher complexity combined with a powerful minoration inequality due to Srebro, Sridharan and Tewari.  As an application we derive a state-of-the-art generalization bound for multi-class gradient boosting. 
\end{abstract}

\section{Introduction}
Multi-output prediction represents an important class of problems that includes multi-class classification \cite{crammer2001algorithmic}, multi-label classification \cite{tsoumakas2007multi,zhang2013review}, multi-target regression \cite{borchani2015survey}, label distribution learning \cite{geng2016label}, structured regression \cite{cortes2016struct} and others, with a wide range of practical applications \cite{xu2019survey}. 

Our objective is to provide a general framework for establishing guarantees for multiple-output prediction problems. A fundamental challenge in the statistical learning theory of multi-output prediction problems is to obtain bounds which allow for (i) favourable convergence rate with the sample size, and (ii) favourable dependence of the risk on the dimensionality of the output space. Whilst modern applications of multi-output prediction deal with increasingly large data sets, they also incorporate problems where the target dimensionality is increasingly large. For example, the number of categories in multi-label is often of the order of tens of thousands, an emergent problem referred to as \emph{extreme classification} \cite{agrawal2013multi,babbar2017dismec,bhatia2015sparse,jain2019slice}.

Formally, the task of multi-output prediction is to learn a vector-valued function from a labelled training set. A common tool in the theoretical analysis of this problem has been a vector-valued extension of Talagrand's contraction inequality for Lipschitz losses \cite{ledoux2013probability}. Both \cite{maurer2016vector} and \cite{cortes2016struct} established vector-contraction inequalities for Rademacher complexity which gave rise to learning guarantees for multi-output prediction problems with a linear dependence upon the dimensionality of the output space. More recently, \cite{lei2019data} has provided more refined vector-contraction inequalities for both Gaussian and Rademacher complexity. This approach leads to a highly favourable sub-linear dependence upon the output dimensionality, which can even be logarithmic, depending upon the degree of regularisation. These structural results lead to a slow convergence rate $O(n^{-1/2})$. \cite{guermeur2017lp} and \cite{musayeva2019} explore an alternative approach based on covering numbers. \cite{chzhen2017benefits} derived a bound for multi-label classification based upon Rademacher complexities. Each of these bounds give rise to favourable dependence upon the dimensionality of the output space, with a rate of order $O(n^{-1/2})$.

Local Rademacher complexities provide a crucial tool in establishing faster rates of convergence \cite{bousquet2002concentration,bartlett2005local,koltchinskii2006local,lei2016local}. By leveraging local Rademacher complexities, \cite{liu2019learning} have derived guarantees for for multi-class learning with function classes which are linear in an RKHS, building upon their previous margin based guarantees \cite{lei2015multi,li2019multi}. This gives rise to fast rates under suitable spectral conditions.  Fast rates of convergence have also been derived by \cite{xu2016MultiLabel} for multi-label classification with linear function spaces. On the other hand, \cite{chzhen2019classification} have derived fast rates of convergence by exploiting an analogue the margin assumption.

Our objective is provide a general framework for establishing generalization bounds for multi-output prediction, which yield fast rates whenever empirical error is small, and apply to a wide variety of function classes, including ensembles of decision trees. We address this problem by generalising to vector-valued functions a smoothness based approach due to \cite{srebro2010smoothness}. A key advantage of our approach is that it allow us to accommodate a wide variety of multi-output loss functions, in conjunction with a variety of hypothesis classes, making our analytic strategy applicable to a variety of learning tasks. Below we summarise our contributions:
\begin{itemize}[noitemsep,topsep=0pt]
\setlength\itemsep{0.5em}
    \item We give a contraction inequality for the local Rademacher complexity of vector-valued functions (Proposition  \ref{localRademacherComplexityProp}). The main ingredient is a self-bounding Lipschitz condition for multi-output loss functions which holds for several widely used examples.
\item We leverage our localised contraction inequality to give a general upper bound for multi-output learning (Theorem \ref{optimisticRademacherBoundForMultioutputPrediction}), which exhibits fast rates whenever the empirical error is small.
\item We demonstrate the minimax-optimality of our result, both in terms of the number of samples, and the output dimensionality, up to logarithmic factors, in the realizable setting (Theorem \ref{minimaxOptimalityRelizableProblemsHilbertSpaceThm}). 
\item Finally, to demonstrate a concrete use our general result, we derive from it a state-of-the-art bound for ensembles of multi-output decision trees (Theorem \ref{resultForEnsemblesOfDecisionTrees}).  
\end{itemize}

\subsection{Problem setting}
We shall consider multi-output prediction problems in supervised learning. Suppose we have a measurable space $\X$, a label space $\Y$ and an output space $\actionSpace$. We shall assume that there is an unknown probability distribution $\probDistribution$ over random variables $(X,Y)$, taking values in $\X\times \Y$.  
The performance is quantified through a loss function $\lossFunction: \actionSpace \times \Y \rightarrow \R$. 

Let $\measurableMaps(\X,\actionSpace)$ denote the set of measurable functions $f:\X \rightarrow \actionSpace$. The goal of the learner is to obtain $f \in \measurableMaps(\X,\actionSpace)$ such that the corresponding risk $
\risk({f},\probDistribution):= \E_{(X,Y) \sim \probDistribution}[  \lossFunction({f}(X),Y)]$ is as low as possible. The learner selects $f \in \measurableMaps(\X,\actionSpace)$ based upon a sample $\sampleXY:=\{(X_i,Y_i)\}_{i \in [n]}$, where $(X_i,Y_i)$ are independent copies of $(X,Y)$. We let $\empiricalRisk(f,\sample):=n^{-1}\cdot \sum_{i \in [n]}\lossFunction(f(X_i),Y_i)$ denote the empirical risk. When the distribution $\probDistribution$ and the sample $\sampleXY$ are clear from context we shall write $\risk({f})$ in place of $\risk({f},\probDistribution)$ and $\empiricalRisk(f)$ in place of $\empiricalRisk(f,\sample)$. We consider \emph{multi-output} prediction problems in which $\actionSpace \subseteq \R^{\numClasses}$. We let $\|\cdot \|_{\infty}$ denote the max norm on $\R^{\numClasses}$ and for positive integer $m \in \N$ we let $[m]:= \{1,\cdots,m\}$.
%Given a set $\actionSpace \subseteq \R^{\numClasses}$ we let $\diam(\actionSpace):=\sup\{\|u-v\|_{\infty}: u,v \in \actionSpace\}$.

\section{The self-bounding Lipschitz condition}

We introduce the following \emph{self-bounding Lipschitz} condition for multi-output loss functions.
\begin{defn}[Self-bounding Lipschitz condition]\label{selfLipschitzDef} A loss function $\lossFunction: \actionSpace \times \Y \rightarrow \R$ is said to be $(\selfLipschitzConstant,\selfLipschitzExponent)$-\emph{self-bounding Lipschitz} for $\selfLipschitzConstant,\selfLipschitzExponent\geq 0$ if for all $y \in \Y$ and $u,v \in \actionSpace$,
\begin{align*}
\left| \lossFunction(u,y)-\lossFunction(v,y)\right| \leq \selfLipschitzConstant \cdot { \max\{ \lossFunction(u,y),\lossFunction(v,y)\}^{\selfLipschitzExponent}} \cdot \left\| u-v\right\|_{\infty}.    
\end{align*}
\end{defn}

This condition interpolates continuously between a classical Lipschitz condition (when $\theta=0$) and a multi-dimensional analogue of a smoothness condition (when $\theta=1/2$), and will be the main assumption that we use to obtain our results. 

Our motivation for introducing Definition \ref{selfLipschitzDef} is as follows. Firstly, in recent work of \cite{lei2019data} the classical Lipschitz condition with respect to the $\ell_{\infty}$ norm has been utilised to derive multi-class bounds with a favourable dependence upon the number of classes $\numClasses$. The role of the $\ell_{\infty}$ norm is crucial since it prevents the deviations in the loss function from accumulating as the output dimension $\numClasses$ grows. Our goal is to give a general framework which simultanously achieves a favourable dependence upon $n$. Secondly, \cite{srebro2010smoothness} introduced a second-order smoothness condition on the loss function. This condition corresponds to the special case whereby $\numClasses=1$ and $\theta=1/2$.  \cite{srebro2010smoothness} showed that this smoothness condition gives rise to a optimistic bound which gives a fast rate $O(n^{-1})$ in the realizable case. The self-bounding Lipschitz provides a multi-dimensional analogue of this condition when $\selfLipschitzExponent=1/2$ which is intended to yield a favourable dependence upon both the number of samples $n$ and the number of classes $\numClasses$. The results established in Sections \ref{mainResultsSec} and \ref{applicationToGBEnsemblesSec} show that this is indeed the case. Finally, by considering the range of exponents $\selfLipschitzExponent \in [0,1/2]$ we will yield convergence rates ranging from slow $O(n^{-1/2})$ to fast $O(n^{-1})$ in the realizable case. This is reminiscent of the celebrated Tsybakov margin condition \cite{mammen1999} which interpolate between slow and fast rates in the parametric classification setting. Crucially, however, whilst the Tsybakov margin condition \cite{mammen1999} is a condition on the underlying distribution which cannot be verified in practice, the self-bounding Lipschitz condition is a property of a loss function which may be verified analytically by the learner. 

\subsection{Verifying the self-bounding Lipschitz condition}\label{verifyingSec}

We start by giving a collection of results which can be used to verify that a given loss function satisfies the self-bounding Lipschitz condition. The following lemmas are proved in Appendix \ref{selfBoundingLipAppendix}.

\begin{lemma}\label{sufficiencyConditionForSelfBoundingLipschitz} Take any $\selfLipschitzConstant>0$, $\selfLipschitzExponent \in [0,1/2]$. Suppose that $\lossFunction: \actionSpace \times \Y \rightarrow [0,\infty)$ is a loss function such that for any $u \in \actionSpace$, $y \in \Y$, there exists a non-negative differentiable function $\varphi_{u,y}:\R \rightarrow [0,\infty)$ satisfying 
\begin{enumerate}
\item $\varphi_{u,y}(0)=\lossFunction(u,y)$;
\item $\forall t>0$, $\sup_{v :\|u-v\|_{\infty}\leq t}\{\lossFunction(v,y)\}\leq \varphi_{u,y}(t)$.
\item The derivative $\varphi_{u,y}'(t)$ is non-negative on $[0,\infty)$;
\item $\forall t_0,t_1 \in \R$, $|\varphi_{u,y}'(t_1)-\varphi_{u,y}'(t_0)|\leq  \left(\frac{\selfLipschitzConstant}{2}\right)^{\frac{1}{1-\selfLipschitzExponent}} \cdot |t_1-t_0|^{\frac{\selfLipschitzExponent}{1-\selfLipschitzExponent}}$;
\end{enumerate} 
Then $\lossFunction: \actionSpace \times \Y \rightarrow [0,\infty)$ is $(\selfLipschitzConstant,\selfLipschitzExponent)$-{self-bounding Lipschitz}.
\end{lemma}

Lemma \ref{cutOffOfSelfBoundingLipschitzLossIsAlsoSelfBoundingLipschitz} shows that clipping preserves this condition. 

\begin{lemma}\label{cutOffOfSelfBoundingLipschitzLossIsAlsoSelfBoundingLipschitz} Suppose that $\lossFunction: \actionSpace \times \Y \rightarrow [0,\infty)$ is a  $(\selfLipschitzConstant,\selfLipschitzExponent)$-{self-bounding Lipschitz} loss function with $\selfLipschitzConstant>0$, $\selfLipschitzExponent \in [0,1]$. Then the loss $\tilde{\lossFunction}: \actionSpace \times \Y \rightarrow [0,\lossFunctionBound]$ defined by $\tilde{\lossFunction}(u,y) = \min\{\lossFunction(u,y),\lossFunctionBound\}$ is $(\selfLipschitzConstant,\selfLipschitzExponent)$-{self-bounding Lipschitz}.
\end{lemma}

Finally, we note the following monotonicity property which follows straightforwardly from the definition. 

\begin{lemma}\label{monotonicitySelfLipschitzLemma} Suppose that $\lossFunction: \actionSpace \times \Y \rightarrow [0,\lossFunctionBound]$ is a bounded $(\selfLipschitzConstant,\selfLipschitzExponent)$-{self-bounding Lipschitz} loss function with $\selfLipschitzConstant>0$, $\selfLipschitzExponent \in [0,1]$. Then given any $\tilde{\selfLipschitzExponent}\leq \selfLipschitzExponent$, the loss $\lossFunction$ is also  $(\tilde{\selfLipschitzConstant},\tilde{\selfLipschitzExponent})$-{self-bounding Lipschitz} with $\tilde{\selfLipschitzConstant}=\selfLipschitzConstant \cdot \lossFunctionBound^{\selfLipschitzExponent-\tilde{\selfLipschitzExponent}}$.
\end{lemma}
 
These properties can be used to establish the self-bounding Lipschitz condition in practical examples.

\subsection{Examples}\label{lossFuncExamplesSec}
We now demonstrate several examples of multi-output loss functions that satisfy our self-bounding Lipschitz condition. In each of the examples below we shall show that the self-bounding Lipschitz condition is satisfied by applying our sufficient condition (Lemma \ref{sufficiencyConditionForSelfBoundingLipschitz}). Detailed proofs are given in Appendix \ref{selfBoundingLipAppendix}. 

%highlights the scope of our analysis, and may facilitate further applications of our methodology.

\subsubsection{Multi-class losses}
We begin with the canonical multi-output prediction problem of multi-class classification in which $\Y = [\numClasses]$ and $\actionSpace = \R^{\numClasses}$. A popular loss function for the theoretical analysis of multi-class learning is the margin loss \cite{crammer2001algorithmic}. The smoothed analogue of the margin loss was introduced by \cite{srebro2010smoothness} in the one-dimensional setting, and \cite{Li2018} in the multi-class setting.

\begin{example}[Smooth margin losses]\label{marginApproximationLoss} Given $\Y = [\numClasses]$ we define the margin function $m:\actionSpace\times \Y \rightarrow \R$ by $m(u,y):= u_y - \max_ {j \in [\numClasses]\backslash\{y\}}\{u_j\}$. The \emph{zero-one loss} $\lossFunction_{0,1}:\actionSpace \times \Y \rightarrow [0,1]$ is defined by $\lossFunction_{0,1}(u,y)=\one\{m(u,y)\leq 0\}$. Whilst natural, the zero-one loss has the drawback of being discontinuous, which presents an obstacle for deriving guarantees.  For each $\rho>0$, the corresponding margin loss $\lossFunction_{\rho}:\actionSpace \times \Y \rightarrow [0,1]$ is defined by $\lossFunction_{\rho}(u,y)=\one\{m(u,y)\leq \rho \}$. The margin loss $\lossFunction_{\rho}$ is also discontinuous. However, we may define a \emph{smooth margin loss} $\tilde{\lossFunction}_{\rho}:\actionSpace \times \Y \rightarrow [0,1]$ by $\tilde{\lossFunction}_{\rho}(u,y)$
{{
\begin{align*}
:=\begin{cases}1 &\text{ if }m(u,y) \leq 0\\
2\left(\frac{m(u,y)}{\rho}\right)^3 - 3 \left(\frac{m(u,y)}{\rho}\right)^2+1&\text{ if }m(u,y) \in [0,\rho]\\
0 &\text{ if }m(u,y) \geq \rho.\end{cases}
\end{align*}}}
\normalsize
By applying Lemma \ref{sufficiencyConditionForSelfBoundingLipschitz} we can show that $\tilde{\lossFunction}_{\rho}$ is  $(\selfLipschitzConstant,\selfLipschitzExponent)$-\emph{self-bounding Lipschitz} with $\selfLipschitzConstant = 4\sqrt{6}\cdot \rho^{-1}$ and $\selfLipschitzExponent = 1/2$. Moreover, the smooth margin loss satisfies ${\lossFunction}_{0,1}(u,y) \leq \tilde{\lossFunction}_{\rho}(u,y)\leq {\lossFunction}_{\rho}(u,y)$ for $(u,y) \in \actionSpace \times \Y$.
\end{example}

The margin loss plays a central role in learning theory and continues to receive significant attention in the analysis of multi-class prediction \cite{guermeur2017lp, Li2018, musayeva2019}, so it is fortuitous that our self-bounding Lipschitz condition incorporates the smooth margin loss. More importantly, however, the self-bounding Lipschitz condition applies to a variety of other loss functions which have received less attention in statisical learning theory.

One of the most widely used loss functions in practical applications is the \emph{multinomial logistic loss}, also known as the softmax loss. 

\begin{example}[Multinomial logistic loss]\label{multiNomialLoss} Given $\Y = [\numClasses]$, the \emph{multinomial logistic loss} $\lossFunction:\actionSpace \times \Y \rightarrow [0,\infty)$ is defined by \[\lossFunction(u,y) = \log \left(\sum_{j \in [\numClasses]} \exp(u_j-u_y)\right),\]
where $u = (u_j)_{j \in [\numClasses]}$ and $y \in [\numClasses]$. For each $(u,y) \in \actionSpace \times [\numClasses]$ let $A_{u,y}=\sum_{j \in [\numClasses]\backslash \{y\}} \exp(u_j-u_y)$ and define $\varphi_{u,y}(t)= \log\left(1+A_{u,y}\cdot \exp(2t)\right)$. By applying Lemma \ref{sufficiencyConditionForSelfBoundingLipschitz} with $\varphi_{u,y}$ we can show that the multinomial logistic loss $\lossFunction$ is  $(\selfLipschitzConstant,\selfLipschitzExponent)$-\emph{self-bounding Lipschitz} with $\selfLipschitzConstant = 2$ and $\selfLipschitzExponent = 1/2$.
\end{example}

Recently, \cite{lei2019data} emphasized that the multinomial-logistic loss is $2$-Lipschitz with respect to the $\ell_{\infty}$-norm (equivalently, $(2,0)$-self-bounding Lipschitz). This gives rise to a slow rate of order $O(n^{-1/2})$. The fact that the multinomial-logistic loss is also $(2,1/2)$-self bounding can be used to derive more favourable guarantees, as we shall see in Section \ref{mainResultsSec}.

\subsubsection{Multi-label losses}
%... multi-label losses
Multi-label prediction is the challenge of classification in settings where instances may be simultaneously assigned to several categories. In multi-label classification we have $\Y \subseteq \{0,1\}^{\numClasses}$, where $\numClasses$ is the total number possible classes. Whilst $\numClasses$ is often very large, the total number of simultaneous labels is typically much smaller. Hence, we consider the set of $k$-sparse binary vectors $\mathbb{S}(k)= \{(y_j)_{j \in [\numClasses]} \in \{0,1\}^{\numClasses}:\sum_{j \in [\numClasses]}y_j \leq k\}$ denote the set of $k$-sparse vectors, where $k \leq [\numClasses]$. We consider the pick-all-labels loss \cite{menon2019multilabel,reddi2019stochastic}.

\begin{example}[Pick-all-labels]\label{pickAllLabelsMultiLabelLoss} Given $\Y = \mathbb{S}(k)$, the \emph{pick-all-labels loss} $\lossFunction:\actionSpace \times \Y \rightarrow [0,\infty)$ is defined by 
\begin{align*}
\lossFunction(u,y) = \sum_{l \in [\numClasses]} y_l \log \left(\sum_{j \in [\numClasses]} \exp(u_j-u_l)\right), 
\end{align*}
where $u = (u_j)_{j \in [\numClasses]} \in \actionSpace$ and $y = (y_j)_{j \in [\numClasses]} \in \Y$. For each $(u,y) \in \actionSpace\times \Y$ we define $\varphi_{u,y}:\R \rightarrow [0, \infty)$ by $A_{u,y}=\sum_{j \in [\numClasses]\backslash \{l\}} \exp(u_j-u_l)$ and let $\varphi_{u,y}(t):= \sum_{l \in [\numClasses]} y_l \log \left(1+ A_{u,y} \cdot \exp(2t)\right)$. By applying Lemma \ref{sufficiencyConditionForSelfBoundingLipschitz} with $\varphi_{u,y}$ we can show that $\lossFunction$ is  $(\selfLipschitzConstant,\selfLipschitzExponent)$-\emph{self-bounding Lipschitz} with $\selfLipschitzConstant = 2 \sqrt{k}$ and $\selfLipschitzExponent = 1/2$.
\end{example}
Crucially, the constant $\selfLipschitzConstant$ for the pick-all-labels family of losses is a function of the sparsity $k$, rather than the total number of labels. This means that our approach is applicable to multi-label problems with with tens of thousands of labels, as long as the label-vectors are $k$-sparse.

\subsubsection{Losses for multi-target regression}

We now return to the problem of \emph{multi-target regression} in which $\Y = \R^{\numClasses}$ \cite{borchani2015survey}.

\begin{example}[Sup-norm losses]\label{multiOutputRegressionLosses} Given $\kappa$, $\gamma \in [1,2]$ we can define a loss-function $\lossFunction:\actionSpace \times \Y \rightarrow \R$ for \emph{multi-target regression} by setting $\lossFunction(u,y) = \kappa \cdot \|u-y\|_{\infty}^{\gamma}$. By applying Lemma \ref{sufficiencyConditionForSelfBoundingLipschitz} with $\varphi_{u,y}(t)=\kappa \cdot (\|u-y\|_{\infty}+t)^{\gamma}$ we can see that $\lossFunction$ is a $(\selfLipschitzConstant,\selfLipschitzExponent)$-\emph{self-bounding Lipschitz} with $\selfLipschitzConstant = (8 \kappa)^{1-\selfLipschitzExponent}$ and $\selfLipschitzExponent=(\gamma-1)/\gamma$. This yields examples of $(\selfLipschitzConstant,\selfLipschitzExponent)$-\emph{self-bounding Lipschitz} loss functions for all $\selfLipschitzConstant >0$ and $\selfLipschitzExponent \in [0,1/2]$.
\end{example}

With these examples in mind we are ready to present our results.

\section{Main results}\label{mainResultsSec}
In this section we give a general upper bound for multi-output prediction problems under the self-bounding Lipschitz condition. A key tool for proving this result will be a contraction inequality for local Rademacher complexity of vector valued functions given in Section \ref{sec:Prop1}, and which may also be of independent interest.  First, we recall the concept of Rademacher complexity. 
\begin{defn}[Rademacher complexity]\label{rademacherComplexityDef} Let $\Z$ be a measurable space and consider a function class $\G \subseteq \measurableMaps(\Z,\R)$. Given a sequence $\zSequenceSizeN = (z_i) \in \Z^n$ we define the \emph{empirical Rademacher complexity} of $\G$ with respect to $\zSequenceSizeN$ by\footnote{Taking the supremum over finite subsets $\tilde{\G} \subseteq \G$ is required to ensure that the function within the expectation is measurable \cite{talagrand2014upper}. This technicality can typically be overlooked.}
\begin{align*}
\hat{\rademacherComplexity}_{\zSequenceSizeN}\left(\G\right):=  \sup_{\tilde{\G} \subseteq \G: |\tilde{\G}|<\infty}  \E_{\sigmaSequenceSizeN}\left( \sup_{g \in \tilde{\G}} \frac{1}{n} \sum_{i \in [n]} \sigma_i \cdot g(z_i)\right),
\end{align*}
where the expectation is taken over sequences of independent Rademacher random variables $\sigmaSequenceSizeN = (\sigma_i)_{i \in [n]}$ with $\sigma_i \in \{-1,+1\}^n$. For each $n \in \N$, the \emph{worst-case Rademacher complexity} of $\G$ is defined by $\rademacherComplexity_n(\G):= \sup_{\zSequenceSizeN \in \Z^n}\hat{\rademacherComplexity}_{\zSequenceSizeN}(\G)$.
\end{defn}

The Rademacher complexity is defined in the context of real-valued functions. However, in this work we deal with multi-output prediction so we shall focus on function classes $\F \subseteq \measurableMaps(\X,\R^{\numClasses})$. In order to utilise the theory of Rademacher complexity in this context we shall transform function classes $\F \subseteq \measurableMaps(\X,\R^{\numClasses})$ into the projected function classes $\Pi \circ \F \subseteq \measurableMaps(\X\times [\numClasses],\R)$ as follows. Firstly, for each $j \in [\numClasses]$ we define $\pi_j: \R^{\numClasses}\rightarrow \R$ to be the projection onto the $j$-th coordinate. We then define, for each $f \in \measurableMaps(\X,\R^{\numClasses})$, the function $\Pi \circ f: \X\times [\numClasses] \rightarrow \R$ by $(\Pi \circ f) (x,j)= \pi_j(f(x))$. Finally, given $\F \subseteq \measurableMaps(\X,\R^{\numClasses})$ we let $\Pi \circ \F:= \{ \Pi \circ f: f \in \F\}\subseteq \measurableMaps(\X\times [\numClasses],\R)$. 

Our central result is the following relative bound.

\newcommand{\multiOutputomplexityTerm}{\Gamma_{n,\numClasses,\delta}^{\selfLipschitzConstant,\selfLipschitzExponent}(\F)}

\begin{theorem}\label{optimisticRademacherBoundForMultioutputPrediction} Suppose we have a class of multi-output functions $\F \subseteq \measurableMaps(\X,[-\hypothesisClassBound,\hypothesisClassBound]^{\numClasses})$, and a $(\selfLipschitzConstant,\selfLipschitzExponent)$-{self-bounding Lipschitz} loss function $\lossFunction:\actionSpace \times \Y \rightarrow [0,\lossFunctionBound]$ for some  $\hypothesisClassBound,\lossFunctionBound \geq 1$, $\selfLipschitzConstant>0$, $\selfLipschitzExponent \in [0,1/2]$. Take $\delta \in (0,1)$, $n \in \N$ and let 
\begin{align*}
\multiOutputomplexityTerm&:= \left(\selfLipschitzConstant \left( \sqrt{\numClasses} \cdot  \log^{3/2} \left(e\hypothesisClassBound n\numClasses\right)\cdot \rademacherComplexity_{n\numClasses}(\Pi \circ \F) + \frac{1}{\sqrt{n}}\right)\right)^{\frac{1}{1-\selfLipschitzExponent}}\\
&\hspace{2cm}+\frac{\lossFunctionBound}{n}\cdot (\log(1/\delta) +\log (\log n) ).
\end{align*}
There exists numerical constants $C_0,C_1>0$ such that given an i.i.d. sample $\sampleXY$ the following holds with probability at least $1-\delta$ for all $f \in \F$,
\begin{align*}
\risk(f) \leq \empiricalRisk(f) + C_0 \cdot \left( \sqrt{\empiricalRisk(f) \cdot \multiOutputomplexityTerm} +\multiOutputomplexityTerm\right).    
\end{align*}
Moreover, if $f^* \in \argmin_{f \in \F} \{\risk(f)\}$ minimises the risk and $\hat{f} \in \argmin_{f \in \F}\{ \empiricalRisk(f)\}$ minimises the empirical risk, then with probability at least $1-\delta$,
\begin{align*}
\risk(\hat{f}) \leq \risk(f^*) + C_1 \cdot \left( \sqrt{\risk(f^*) \cdot \multiOutputomplexityTerm} +\multiOutputomplexityTerm\right).
\end{align*}
\end{theorem}

 The proof of Theorem \ref{optimisticRademacherBoundForMultioutputPrediction} is built upon a local contraction inequality result (Proposition \ref{localRademacherComplexityProp}, Section \ref{sec:Prop1}). The result follows by combining with techniques from \cite{bousquet2002concentration}. For details see Appendix \ref{proofOfMainOptimisticBoundThmSec}.

Theorem \ref{optimisticRademacherBoundForMultioutputPrediction} gives an upper bound for the generalization gap $(\risk(f) - \empiricalRisk(f))$, framed in terms of a complexity term $\multiOutputomplexityTerm$, which depends upon both the Rademacher complexity of the projected function class  $\rademacherComplexity_{n\numClasses}(\Pi\circ \F)$ and the self-bounding Lipschitz parameters $\selfLipschitzConstant$, $\selfLipschitzExponent$. When the empirical error is small in relation to the complexity term ($\empiricalRisk(f) \leq  \multiOutputomplexityTerm$), the generalization gap is of order $\multiOutputomplexityTerm$. In less favourable circumstances we recover a bound of order $\multiOutputomplexityTerm$. 

In Section \ref{minimaxOptimalitySec} we will demonstrate that in the realizable setting, Theorem \ref{optimisticRademacherBoundForMultioutputPrediction} is minimax optimal up to logarithmic factors, both in terms of the sample size $n$, and the output dimension $\numClasses$. In Section \ref{applicationToGBEnsemblesSec} we will demonstrate that Theorem \ref{optimisticRademacherBoundForMultioutputPrediction} yields state of the art guarantees for ensembles of decision trees for multi-output prediction problems.

\subsection{Comparison with state of the art}

In this section we compare our main result (Theorem \ref{optimisticRademacherBoundForMultioutputPrediction}) with a closely related guarantee due to \cite{lei2019data}. We say that a loss function $\lossFunction$ is $\selfLipschitzConstant$-Lipschitz if it is $(\selfLipschitzConstant,\selfLipschitzExponent)$-self-bounding Lipschitz with $\selfLipschitzExponent=0$.

\newcommand{\additiveMultiOutputomplexityTerm}{\mathfrak{J}_{n,\numClasses,\delta}^{\selfLipschitzConstant}(\F)}

\begin{theorem}\cite{lei2019data}\label{additiveRademacherBoundForMultioutputPrediction} Suppose we have a class of multi-output functions $\F \subseteq \measurableMaps(\X,[-\hypothesisClassBound,\hypothesisClassBound]^{\numClasses})$, and a $\selfLipschitzConstant$-{Lipschitz} loss function $\lossFunction:\actionSpace \times \Y \rightarrow [0,\lossFunctionBound]$ for some  $\hypothesisClassBound,\lossFunctionBound\geq 1$ and $\selfLipschitzConstant >0$. Take $\delta \in (0,1)$, $n \in \N$ and let 
\begin{align*}
\additiveMultiOutputomplexityTerm:= \selfLipschitzConstant \left( \sqrt{\numClasses} \cdot  \log^{3/2} \left(e\hypothesisClassBound n\numClasses\right)\cdot \rademacherComplexity_{n\numClasses}(\Pi \circ \F)+\frac{1}{\sqrt{n}}\right).    
\end{align*}
There exists numerical constants $C_2,C_3>0$ such that given an i.i.d. sample $\sampleXY$ the following holds with probability at least $1-\delta$ for all $f \in \F$, 
\[\risk(f) \leq \empiricalRisk(f) + C_2 \cdot\additiveMultiOutputomplexityTerm +\lossFunctionBound\sqrt{\frac{\log(1/\delta)}{n}}.\]
Moreover, if $f^* \in \argmin_{f \in \F} \{\risk(f)\}$ minimises the risk and $\hat{f} \in \argmin_{f \in \F}\{ \empiricalRisk(f)\}$ minimises the empirical risk, then with probability at least $1-\delta$,
\[\risk(\hat{f}) \leq \risk(f^*) + C_3 \cdot\additiveMultiOutputomplexityTerm+2\lossFunctionBound\sqrt{\frac{\log(1/\delta)}{n}}.\] 
\end{theorem}

Theorem \ref{additiveRademacherBoundForMultioutputPrediction} is a mild generalization of Theorem 6 from \cite{lei2019data}, which establishes the special case of Theorem \ref{additiveRademacherBoundForMultioutputPrediction} in which $\F$ is an RKHS and the learning problem is multi-class classification. For completeness we show that Theorem \ref{additiveRademacherBoundForMultioutputPrediction} follows from Proposition \ref{localRademacherComplexityProp} in Appendix \ref{proofOfMainOptimisticBoundThmSec}. Note that by the monotonicity property (Lemma \ref{monotonicitySelfLipschitzLemma}) any loss function $\lossFunction: \actionSpace \times \Y \rightarrow [0,\lossFunctionBound]$ which is $(\selfLipschitzConstant,\selfLipschitzExponent)$-self-bounding Lipschitz is also $\selfLipschitzConstant \cdot \lossFunctionBound^{\selfLipschitzExponent}$-Lipschitz, so the additve bound in Theorem \ref{additiveRademacherBoundForMultioutputPrediction} also applies.

To gain a deeper intuition for the bound in Theorem \ref{optimisticRademacherBoundForMultioutputPrediction} we compare with the bound in Theorem \ref{additiveRademacherBoundForMultioutputPrediction}. Let's suppose that $\rademacherComplexity_{n\numClasses}(\Pi\circ \F) = \tilde{O}((n\numClasses)^{-1/2})$ (for a concrete example where this is the case see Section \ref{applicationToGBEnsemblesSec}). We then have $\multiOutputomplexityTerm = \tilde{O}( n^{-\frac{1}{2(1-\selfLipschitzExponent)}})$. For large values of  $\empiricalRisk(f)$ Theorem \ref{optimisticRademacherBoundForMultioutputPrediction} gives a bound on generalization gap  $(\risk(f) - \empiricalRisk(f))$ of order $\tilde{O}( n^{-\frac{1}{4(1-\selfLipschitzExponent)}})$, which is slower than the rate achieved by Theorem \ref{additiveRademacherBoundForMultioutputPrediction} whenever $\selfLipschitzExponent<1/2$. However, when $\empiricalRisk(f)$ is small ($\empiricalRisk(f) \leq \tilde{O}( n^{-\frac{1}{2(1-\selfLipschitzExponent)}})$), Theorem \ref{optimisticRademacherBoundForMultioutputPrediction} gives rise to a bound of order $\tilde{O}( n^{-\frac{1}{2(1-\selfLipschitzExponent)}})$, yielding faster rates than can be obtained through the standard Lipschitz condition alone whenever $\selfLipschitzExponent>0$. Finally note that if the loss $\lossFunction$ is $(\selfLipschitzConstant,\selfLipschitzExponent)$-self-bounding Lipschitz with $\selfLipschitzExponent=1/2$ then the rates given by Theorem \ref{optimisticRademacherBoundForMultioutputPrediction} always either match or outperform the rates given by Theorem \ref{additiveRademacherBoundForMultioutputPrediction}. Moreover, $\selfLipschitzExponent=1/2$ occurs for several practical examples discussed in Section \ref{lossFuncExamplesSec} including the multinomial-logistic loss.

\subsection{A contraction inequality for the local Rademacher compliexity of vector-valued function classes}\label{sec:Prop1}

We now turn to stating and proving the key ingredient of our main result, Proposition \ref{localRademacherComplexityProp}. First we introduce some additional notation.

Suppose $f \in \measurableMaps(\X,\actionSpace)$. Given a loss function $\lossFunction:\actionSpace\times \Y \rightarrow \R$ we define $\lossFunction \circ f:\X\times \Y \rightarrow \R$ by  $(\lossFunction \circ f)(x,y)=\lossFunction(f(x),y)$. We extend this definition to function classes $\F \subseteq \measurableMaps(\X,\actionSpace)$ by $\lossFunction \circ \F = \{\lossFunction \circ f: f \in \F \}$. Moreover, for each $\xySequenceSizeN \in (\X \times \Y)^n$ and $r>0$, a subset $ \F|^r_{\xySequenceSizeN}:= \{ f \in \F:\hspace{2mm}\empiricalRisk(f,\zSequenceSizeN)\leq r \}$. Intuitively, the local Rademacher complexity allows us to \emph{zoom in} upon the neighbourhood of the empirical risk minimizer. This is the subset that matters in practice and is typically much smaller than the full $\Pi\circ \F$.

\begin{prop}\label{localRademacherComplexityProp} Suppose we have a class of multi-output functions $\F \subseteq \measurableMaps(\X,[-\hypothesisClassBound,\hypothesisClassBound]^{\numClasses})$, where $\hypothesisClassBound \geq 1$. Given a $(\selfLipschitzConstant,\selfLipschitzExponent)$-{self-bounding Lipschitz} loss function $\lossFunction:\actionSpace \times \Y \rightarrow [0,\R]$, where $\selfLipschitzConstant>0$, $\selfLipschitzExponent \in [0,1/2]$ and $\xySequenceSizeN \in (\X \times \Y)^n$, $r>0$, we have, 
\begin{align*}
\hat{\rademacherComplexity}_{\xySequenceSizeN}\left(\lossFunction\circ\F|^r_{\xySequenceSizeN} \right)\leq \selfLipschitzConstant r^{\selfLipschitzExponent}  \left( 2^9 \sqrt{\numClasses} \cdot  \log^{3/2} \left(e\hypothesisClassBound n\numClasses\right)\cdot \rademacherComplexity_{n\numClasses}(\Pi \circ \F) +n^{-1/2}\right).
\end{align*}
\end{prop}

The proof of Proposition \ref{localRademacherComplexityProp}, given later in this section, relies upon covering numbers.

\begin{defn}[Covering numbers]\label{coveringNumbersDef} Let $(\metricSpace,\dist)$ be a semi-metric space. Given a set $A \subseteq \metricSpace$ and an $\epsilon>0$, a subset $\tilde{A}\subseteq A$ is said to be a (proper) $\epsilon$-cover of $A$ if, for all $a \in A$, there exists some $\tilde{a} \in \tilde{A}$ with $\dist(a,\tilde{a})\leq \epsilon$. We let $\coveringNumber(\epsilon,A,\dist)$ denote the minimal cardinality of an $\epsilon$-cover for $A$.
\end{defn}

We shall consider covering numbers for two classes of data-dependent semi-metric spaces. Let $\Z$ be a measurable space and take $\G \subseteq  \measurableMaps(\Z,\R)$. For each $n \in \N$ and each sequence  $\zSequenceSizeN=(z_i)_{i \in [n]} \in \Z^n$ we define a pair of metrics $\dist_{\zSequenceSizeN,2}$ and $\dist_{\zSequenceSizeN,\infty}$ by 
\begin{align*}
\dist_{\zSequenceSizeN,2}(g_0,g_1)&:= \sqrt{\frac{1}{n}  \sum_{i \in [n]}(g_0(z_i)-g_1(z_i))^2}\\
\dist_{\zSequenceSizeN,\infty}(g_0,g_1)&:=\max_{i \in [n]}\{|g_0(z_i)-g_1(z_i)|\},
\end{align*}
where $g_0,g_1\in \G$. The first stage of the proof of Proposition \ref{localRademacherComplexityProp} will be using the following lemma which bounds the covering number of $\lossFunction\circ\F|^r_{\xySequenceSizeN} $ in terms of an associated covering number for $\Pi(\F)$.
\begin{lemma}\label{keyCoveringNumberBoundInProofOfLocalRadCompProp}  Suppose that $\F \subseteq \measurableMaps(\X,\R^{\numClasses})$ and $\lossFunction$ is $(\selfLipschitzConstant,\selfLipschitzExponent)$-{self-bounding Lipschitz} with $\selfLipschitzExponent \in [0,1/2]$. Take $\lossFunction:\actionSpace \times \Y \rightarrow [0,\lossFunctionBound]$, $\xySequenceSizeN = \{(x_i,y_i)\}_{i \in [n]} \in (\X \times \Y)^n$, $r>0$ and define $\wSequenceSizeNQ = \{(x_i,j)\}_{(i,j) \in [n]\times [\numClasses]} \in (\X\times [\numClasses])^{n \numClasses}$. Given any $f_0, f_1 \in \F|_{\xySequenceSizeN}^r$,
\[\dist_{\xySequenceSizeN,2}(\lossFunction \circ f_0, \lossFunction \circ {f}_1) \leq  2^{\selfLipschitzExponent} \selfLipschitzConstant r^{\selfLipschitzExponent} \cdot \dist_{\wSequenceSizeNQ,\infty}(\Pi \circ f_0, \Pi \circ f_1).\]

Moreover, for any $\epsilon>0$,
$\coveringNumber\left(2^{1+\selfLipschitzExponent} \selfLipschitzConstant r^{\selfLipschitzExponent} \cdot \epsilon,\lossFunction\circ\F|^r_{\xySequenceSizeN} ,\dist_{\xySequenceSizeN,2}\right) \leq \coveringNumber\left( \epsilon,\Pi \circ \F,\dist_{\wSequenceSizeNQ,\infty}\right)$.
\end{lemma}
\begin{proof}[Proof of Lemma \ref{keyCoveringNumberBoundInProofOfLocalRadCompProp}] To prove the first part of the lemma we take $f_0, f_1 \in \F|_{\xySequenceSizeN}^r$ and let $\zeta = \dist_{\wSequenceSizeNQ,\infty}(\Pi \circ f_0, \Pi \circ f_1)$. It follows from the construction of $\wSequenceSizeNQ$ that $|\pi_j(f_0(x_i))-\pi_j({f}_1(x_i))| \leq  \zeta$ for each $(i,j) \in [n]\times [\numClasses]$, so $\|f_0(x_i)-{f}_1(x_i)\|_{\infty} \leq \zeta$ for each $i \in [n]$. 

Furthermore, by the self-bounding Lipschitz condition we deduce that for each $i \in [n]$,
{{\begin{align*}
 |\lossFunction(f_0(x_i),y_i)-\lossFunction({f}_1(x_i),y_i)|& \leq \selfLipschitzConstant \cdot \max\left\lbrace \lossFunction(f_0(x_i),y_i),\lossFunction({f}_1(x_i),y_i)\right\rbrace^{\selfLipschitzExponent}\cdot \|f_0(x_i)-{f}_1(x_i)\|_{\infty}\\
& \leq  \selfLipschitzConstant \cdot \max\left\lbrace \lossFunction(f_0(x_i),y_i),\lossFunction({f}_1(x_i),y_i)\right\rbrace^{\selfLipschitzExponent} \cdot \zeta.
\end{align*}}}
Hence, by Jensen's inequality we have
{{
\begin{align*}
\dist_{\xySequenceSizeN,2}(\lossFunction \circ f_0, \lossFunction \circ {f}_1)^2 
&= \frac{1}{n}\sum_{i \in [n]}\left(\lossFunction(f_0(x_i),y_i)-\lossFunction({f}_1(x_i),y_i)\right)^2\\
& \leq (\selfLipschitzConstant \zeta)^{2} \cdot \frac{1}{n}\sum_{i \in [n]}\max\left\lbrace \lossFunction(f_0(x_i),y_i),\lossFunction({f}_1(x_i),y_i)\right\rbrace^{2\selfLipschitzExponent}\\
& \leq (\selfLipschitzConstant \zeta)^{2} \cdot \left(\frac{1}{n}\sum_{i \in [n]}\max\left\lbrace \lossFunction(f_0(x_i),y_i),\lossFunction({f}_1(x_i),y_i)\right\rbrace\right)^{2\selfLipschitzExponent}\\
& \leq (\selfLipschitzConstant \zeta)^{2} \cdot \left( \empiricalRisk(f_0,\xySequenceSizeN) +\empiricalRisk({f}_1,\xySequenceSizeN) \right)^{2\selfLipschitzExponent} \leq  (
\selfLipschitzConstant \zeta)^{2}  \cdot (2r)^{2\selfLipschitzExponent},
\end{align*}}}
where we use the fact that  $\selfLipschitzExponent \in [0,1/2]$ and $\max\{ \empiricalRisk(f_0,\zSequenceSizeN) ,\empiricalRisk({f}_1,\zSequenceSizeN)   \}\leq r$. Thus, 
\begin{align*}
\dist_{\xySequenceSizeN,2}(\lossFunction \circ f_0, \lossFunction \circ {f}_1) &\leq  2^{\selfLipschitzExponent} \selfLipschitzConstant r^{\selfLipschitzExponent} \cdot \zeta=  2^{\selfLipschitzExponent} \selfLipschitzConstant r^{\selfLipschitzExponent} \cdot \dist_{\wSequenceSizeNQ,\infty}(\Pi \circ f_0, \Pi \circ f_1).
\end{align*}
This completes the proof of the first part of the lemma.

To prove the second part of the lemma we note that since $\Pi\circ \F|^r_{\xySequenceSizeN} \subseteq \Pi \circ \F$ we have\footnote{The factor of $2$ is required as we are using \emph{proper} covers, which are subsets of the set being covered (see Definition \ref{coveringNumbersDef}).} 
\begin{align*}
  \coveringNumber\left( 2\epsilon,\Pi \circ \F|^r_{\xySequenceSizeN},\dist_{\wSequenceSizeNQ,\infty}\right) \leq  \coveringNumber\left( \epsilon,\Pi \circ \F,\dist_{\wSequenceSizeNQ,\infty}\right),  
\end{align*}
so we may choose $f_1,\cdots,f_m \in  \F|^r_{\xySequenceSizeN}$ with $m\leq \coveringNumber\left( \epsilon,\Pi \circ \F,\dist_{\wSequenceSizeNQ,\infty}\right)$ such that $\Pi \circ f_1,\cdots, \Pi \circ f_m$ forms a $2\epsilon$-cover of $\Pi \circ \F|^r_{\xySequenceSizeN}$ with respect to the $\dist_{\wSequenceSizeNQ,\infty}$ metric. 

To complete the proof it suffices to show that $\lossFunction \circ f_1,\cdots, \lossFunction \circ f_m$ is a $2^{1+\selfLipschitzExponent} \selfLipschitzConstant r^{\selfLipschitzExponent} \cdot \epsilon$-cover of $\lossFunction\circ\F|^r_{\xySequenceSizeN}$ with respect to the $\dist_{\xySequenceSizeN,2}$ metric. 

Take any $\tilde{g} \in \lossFunction\circ\F|^r_{\xySequenceSizeN}$, so $\tilde{g}= \lossFunction \circ \tilde{f}$ for some $\tilde{f} \in \F|^r_{\xySequenceSizeN}$. Since $\Pi \circ f_1,\cdots, \Pi \circ f_m$ forms a $2\epsilon$-cover of $\Pi \circ \F|^r_{\xySequenceSizeN}$ we may choose $l \in [m]$ so that $\dist_{\wSequenceSizeNQ,\infty}(\Pi \circ f_l, \Pi \circ \tilde{f}) \leq 2\epsilon$. By the first part of the lemma we deduce that 
\begin{align*}
 \dist_{\xySequenceSizeN,2}(\lossFunction \circ f_l,  \tilde{g})=\dist_{\xySequenceSizeN,2}(\lossFunction \circ f_l, \lossFunction \circ \tilde{f}) \leq  2^{1+\selfLipschitzExponent} \selfLipschitzConstant r^{\selfLipschitzExponent} \cdot \epsilon   
\end{align*}
%$$.
Since this holds for all $\tilde{g} \in \lossFunction\circ\F|^r_{\xySequenceSizeN}$, we see that $\lossFunction \circ f_1,\cdots, \lossFunction \circ f_m$ is a $2^{1+\selfLipschitzExponent} \selfLipschitzConstant r^{\selfLipschitzExponent} \cdot \epsilon$-cover of $\lossFunction\circ\F|^r_{\xySequenceSizeN}$, which completes the proof of the lemma.
\end{proof}

To prove Proposition \ref{localRademacherComplexityProp}, we shall also utilise two technical results to move from covering numbers to Rademacher complexity and back. First, we shall use the following powerful result from \cite{srebro2010smoothness} which gives an upper bound for \emph{worst-case} covering numbers in terms of the \emph{worst-case} Rademacher complexity.
\begin{theorem}[\cite{srebro2010smoothness}]\label{srebroMinorationBound} Given a measurable space $\Z$ and a function class $\G \subseteq \measurableMaps(\Z,[-\hypothesisClassBound,\hypothesisClassBound])$, any $\epsilon> 2\cdot \rademacherComplexity_n(\G)$ and any $\zSequenceSizeN \in \Z^n$,
\begin{align*}
\log \coveringNumber(\epsilon,\G,\dist_{\zSequenceSizeN,\infty}) \leq  \left(\rademacherComplexity_n(\G)\right)^2 \cdot \frac{4n}{\epsilon^{2}}\cdot \log \frac{2e\hypothesisClassBound n}{\epsilon}.
\end{align*}
\end{theorem}
 We can view this result as an analogue of Sudakov's minoration inequality for $\ell_{\infty}$ covers, rather than $\ell_2$ covers.

Secondly, we shall use Dudley's inequality \cite{dudley1967sizes} which allows us to bound Rademacher complexities in terms of covering numbers. We shall use the following variant due to \cite{guermeur2017lp} as it yields more favourable constants.

\begin{theorem}[\cite{guermeur2017lp}]\label{guermeursDudleyTypeBound} Suppose we have  a measurable space $\Z$, a function class $\G \subseteq \measurableMaps(\Z,\R)$ and a sequence $\zSequenceSizeN \in \Z^n$. For any decreasing sequence $(\epsilon_k)_{k=0}^{\infty}$ with $\underset{ k\rightarrow \infty}{\lim}\epsilon_k=0$ with $\epsilon_0\geq \sup_{g_0,g_1 \in \G}\dist_{\xySequenceSizeN,2}(g_0,g_1)$, the following inequality holds for all $K \in \N$,
{
\begin{align*}
\hat{\rademacherComplexity}_{\zSequenceSizeN}(\G) \leq 2 \cdot \sum_{k=1}^K(\epsilon_k+\epsilon_{k-1}) \cdot \sqrt{\frac{\log \coveringNumber(\epsilon_k,\G,\dist_{\xySequenceSizeN,2})}{n}}+\epsilon_K.
\end{align*}}
\end{theorem}

We are now ready to complete the proof of our local Rademacher complexity inequality.
\begin{proof}[Proof of Proposition \ref{localRademacherComplexityProp}] Take $\xySequenceSizeN = \{(x_i,y_i)\}_{i \in [n]} \in (\X \times \Y)^n$ and $r>0$ and define $\wSequenceSizeNQ = \{(x_i,j)\}_{(i,j) \in [n]\times [\numClasses]} \in (\X\times [\numClasses])^{n \numClasses}$. By Lemma \ref{keyCoveringNumberBoundInProofOfLocalRadCompProp} combined with Theorem \ref{srebroMinorationBound} applied to $\Pi \circ \F$ we see that for each $\xi> 2\cdot \rademacherComplexity_{n\numClasses}(\Pi\circ \F)$ we have
\begin{align}\label{coveringNumIneqProofOfPropLocalRadIneq}
\log \coveringNumber\left(2^{1+\selfLipschitzExponent} \selfLipschitzConstant r^{\selfLipschitzExponent} \cdot \xi,\lossFunction\circ\F|^r_{\xySequenceSizeN} ,\dist_{\xySequenceSizeN,2}\right)\nonumber & \leq 
\log \coveringNumber(\xi,\Pi \circ \F,\dist_{\wSequenceSizeNQ,\infty}) \nonumber\\ &\leq  \left(\rademacherComplexity_{n\numClasses}(\Pi \circ \F)\right)^2 \cdot \frac{4n\numClasses}{\xi^{2}}\cdot \log \frac{2e\hypothesisClassBound n\numClasses}{\xi}.
\end{align}
Moreover, given any $g_0 = \lossFunction \circ f_0$, $g_1 = \lossFunction \circ f_1 \in \lossFunction\circ\F|^r_{\xySequenceSizeN}$, so $\dist_{\wSequenceSizeNQ,\infty}(\Pi \circ f_0, \Pi \circ f_1)\leq 2\hypothesisClassBound$, so by the first part of Lemma \ref{keyCoveringNumberBoundInProofOfLocalRadCompProp} we have $\dist_{\xySequenceSizeN,2}(g_0,g_1)  \leq 2^{1+\selfLipschitzExponent} \selfLipschitzConstant r^{\selfLipschitzExponent} \cdot \hypothesisClassBound$.

Now construct $(\epsilon_k)_{k =0}^{\infty}$ by $\epsilon_k = 2^{1+\selfLipschitzExponent} \selfLipschitzConstant r^{\selfLipschitzExponent} \cdot \hypothesisClassBound \cdot 2^{-k}$  and choose 
\begin{align*}
K = \lceil \log_2\left(\hypothesisClassBound\cdot  \min\{(2\cdot \rademacherComplexity_{n\numClasses}(\Pi \circ \F))^{-1}, (8\sqrt{n})\}\right)  \rceil -1    
\end{align*} 
  Hence, $\sup_{g_0,g_1 \in \Pi \circ \F|^r_{\xySequenceSizeN}} \dist_{\zSequenceSizeN,2}(g_0,g_1) \leq \epsilon_0$ and $\hypothesisClassBound \cdot 2^{-K-1} \leq \max\{2 \cdot \rademacherComplexity_{n\numClasses}(\Pi \circ \F),(8\sqrt{n})^{-1}\} < \hypothesisClassBound \cdot 2^{-K}$. 
  
  Furthermore, for $k \leq K$ by letting $\xi_k = \hypothesisClassBound \cdot 2^{-k}$, we have $\epsilon_k = 2^{1+\selfLipschitzExponent} \selfLipschitzConstant r^{\selfLipschitzExponent} \cdot \xi_k$ and $\xi_k > \max\{2 \cdot \rademacherComplexity_{n\numClasses}(\Pi \circ \F), (8\sqrt{n})^{-1}\}$, so by eq. (\ref{coveringNumIneqProofOfPropLocalRadIneq}) 
\begin{align*}
\log \coveringNumber\left( \epsilon_k,\lossFunction\circ\F|^r_{\xySequenceSizeN} ,\dist_{\xySequenceSizeN,2}\right)\nonumber 
&\leq  \left(\rademacherComplexity_{n\numClasses}(\Pi \circ \F)\right)^2 \cdot \frac{4n\numClasses}{\xi_k^{2}}\cdot \log \frac{2e\hypothesisClassBound n\numClasses}{\xi_k}\\
&\leq \left( 2^{1+\selfLipschitzExponent} \selfLipschitzConstant r^{\selfLipschitzExponent}  \cdot  \rademacherComplexity_{n\numClasses}(\Pi \circ \F)\right)^2 \cdot \frac{4n\numClasses}{\epsilon_k^{2}}\cdot \log \left(e\hypothesisClassBound (n\numClasses)^{3/2}\right)\\
&\leq \left( 2^{1+\selfLipschitzExponent} \selfLipschitzConstant r^{\selfLipschitzExponent}  \cdot  \rademacherComplexity_{n\numClasses}(\Pi \circ \F)\right)^2 \cdot \frac{6n\numClasses}{\epsilon_k^{2}}\cdot \log \left(e\hypothesisClassBound n\numClasses\right).
\end{align*}
Note also that by construction $K \leq 4\log(e\hypothesisClassBound n\numClasses)$. 

By Theorem \ref{guermeursDudleyTypeBound} and $\epsilon_{k-1}=2\cdot \epsilon_k$ we deduce that 
\begin{align*}
\hat{\rademacherComplexity}_{\xySequenceSizeN}(\lossFunction \circ \F|_{\xySequenceSizeN}^r)
&\leq 2 \cdot \sum_{k=1}^K(\epsilon_k+\epsilon_{k-1}) \cdot \sqrt{\frac{\log \coveringNumber(\epsilon_k,\lossFunction\circ\F|^r_{\xySequenceSizeN},\dist_{\xySequenceSizeN,2})}{n}}+\epsilon_K\\
& \leq 6 \sum_{k=1}^K \epsilon_k \cdot \sqrt{\frac{\log \coveringNumber(\epsilon_k,\lossFunction\circ\F|^r_{\xySequenceSizeN},\dist_{\xySequenceSizeN,2})}{n}}+\epsilon_K\\
& \leq 6K \cdot \left( 2^{1+\selfLipschitzExponent} \selfLipschitzConstant r^{\selfLipschitzExponent}  \cdot  \rademacherComplexity_{n\numClasses}(\Pi \circ \F)\right) \cdot \sqrt{ 6\numClasses \cdot \log \left(e\hypothesisClassBound n\numClasses\right)}+\epsilon_K\\
& \leq 2^8 \sqrt{\numClasses} \cdot \left(\selfLipschitzConstant r^{\selfLipschitzExponent}  \cdot  \rademacherComplexity_{n\numClasses}(\Pi \circ \F)\right) \cdot \log^{3/2} \left(e\hypothesisClassBound n\numClasses\right)+\epsilon_K\\
& \leq \selfLipschitzConstant r^{\selfLipschitzExponent}  \left( 2^9 \sqrt{\numClasses} \cdot  \log^{3/2} \left(e\hypothesisClassBound n\numClasses\right)\cdot \rademacherComplexity_{n\numClasses}(\Pi \circ \F) +n^{-1/2}\right).
\end{align*}
This completes the proof of the proposition.
\end{proof}

\section{Minimax optimality}\label{minimaxOptimalitySec}

\newcommand{\minimaxRiskReliazableSelfBoundingLipschitz}{\mathfrak{M}(\selfLipschitzConstant,\selfLipschitzExponent,n,\numClasses,\kappa)}

In this section we investigate the optimality of our generalization guarantees. Theorem \ref{optimisticRademacherBoundForMultioutputPrediction} gives a rate of order $O(n^{-\frac{1}{2(1-\selfLipschitzExponent)}})$ when $\rademacherComplexity_{n \numClasses}(\Pi \circ \F)=O(n^{-\frac{1}{2}})$ and $\risk(f^*)=0$. It is natural ask whether this rate can be improved upon. Moreover, we have good reason to be suspicious since in the parametric case, where the covering numbers of $\F$ grow logarithmically with $\epsilon$ (eg. function classes of finite psuedo-dimension), one can obtain rates of order $O(n^{-1})$, even when the loss function is Lipschitz ($\selfLipschitzExponent=0$) \cite{bartlett2005local, lei2016local}. Hence, Theorem \ref{optimisticRademacherBoundForMultioutputPrediction} is sub-optimal for parametric function classes. However, it turns out that Theorem \ref{optimisticRademacherBoundForMultioutputPrediction} is \emph{minimax optimal} in the non-parametric setting, as we shall now show.

Throughout this section we shall focus on \emph{multi-target regression} problems on an infinite dimensional space. More precisely, throughout this section we take $\X $ to be an arbitrary infinite space (eg. $\X = \N$) and take $\actionSpace= \Y = [-1,1]^{\numClasses}$ for some $\numClasses \in \N$.

\begin{defn}[Realizable problems]\label{realizableProblems} Given a loss function $\lossFunction: \actionSpace \times \Y \rightarrow [0,\infty)$ and a function class $\F \subseteq \measurableMaps(\X,\actionSpace)$, a probability distribution $\probDistribution$ on $\X\times \Y$ is said to be a $(\lossFunction,\F)$-\emph{realizable problem} if there exists some $f^* \in \F$ satisfying $\risk(f^*,\probDistribution) = 0$.
\end{defn}

In this section we study the minimax risk over the class of realizable problems. 

\begin{defn}[Maximal minimax risk]\label{minimaxRiskRealizableProblems} Given $n, \numClasses \in \N$, $\kappa,\selfLipschitzConstant,\selfLipschitzExponent>0$,
\begin{align*}
\minimaxRiskReliazableSelfBoundingLipschitz:=\sup_{\lossFunction,\F}\left\lbrace \inf_{\hat{\phi}}\left\lbrace \sup_{\probDistribution}\left\lbrace \E_{\sample}\left[ \risk(\hat{\phi}_{\sample},\probDistribution)   \right] \right\rbrace \right\rbrace\right\rbrace,
\end{align*}
where the first supremum ranges over all $(\selfLipschitzConstant,\selfLipschitzExponent)$-self-bounding Lipschitz loss functions $\lossFunction:\actionSpace\times \Y \rightarrow [0,1]$ and function classes $\F\subseteq \measurableMaps(\X,\actionSpace)$ satisfying $\rademacherComplexity_{n\numClasses}(\Pi \circ \F) \leq \sqrt{ \kappa / (n \numClasses)}$, the infimum ranges over all algorithms $\hat{\phi}$ which take a sample $\sample =  \{(X_i,Y_i)\}_{i \in [n]} \in (\X\times \Y)^n$ and output a function $\hat{\phi}_{\sample} \in \F$, and the second supremum ranges over all  $(\lossFunction,\F)$-realizable problems $\probDistribution$. 
\end{defn}
Intuitively, the minimax risk gives the best possible expected risk that may be obtained by a learning algorithm, uniformly, over a class of learning problems.

\begin{theorem}\label{minimaxOptimalityRelizableProblemsHilbertSpaceThm} There exists a numerical constant $C_4\geq 1$ such that for any $\selfLipschitzConstant\geq 1$, $\selfLipschitzExponent \in [0,1/2]$, $n$, $\numClasses \in \N$ and $\kappa \in [1, {n}/\selfLipschitzConstant^{2}]$,
\begin{align*}
C_4^{-1} \left(\selfLipschitzConstant   \sqrt{\frac{\kappa}{n}}\right)^{\frac{1}{1-\selfLipschitzExponent}} \leq \minimaxRiskReliazableSelfBoundingLipschitz \leq  C_4 \log^{3}(en\numClasses) \left(\selfLipschitzConstant  \sqrt{\frac{\kappa}{n}}\right)^{\frac{1}{1-\selfLipschitzExponent}}.    
\end{align*}
\end{theorem}
The proof of Theorem \ref{minimaxOptimalityRelizableProblemsHilbertSpaceThm} consists of an upper bound and a lower bound. The upper bound is a straightforward consequence of Theorem \ref{optimisticRademacherBoundForMultioutputPrediction}. For the lower bound we adapt a classical argument of \cite{ehrenfeucht1989general} with a construction using the loss function given in Example \ref{multiOutputRegressionLosses}. A full proof is presented in Appendix \ref{pfOfMinimaxRateSec}.

\subsection{Optimality of the exponent range}

We close this section by considering the optimality with respect to the \emph{range} of possible exponents in our generalization bounds. Theorem \ref{optimisticRademacherBoundForMultioutputPrediction} presupposes that $\selfLipschitzExponent \in [0,1/2]$. This is required for the proof at two stages. Firstly, the application of Jensen's inequality in the proof of Lemma \ref{keyCoveringNumberBoundInProofOfLocalRadCompProp} requires the function $z\mapsto z^{2\selfLipschitzExponent}$ to be concave, which is no longer true if $\selfLipschitzExponent>1/2$. Secondly, even if we could establish Proposition \ref{localRademacherComplexityProp} for $\selfLipschitzExponent >1/2$, deducing Theorem \ref{optimisticRademacherBoundForMultioutputPrediction} from Proposition \ref{localRademacherComplexityProp} requires that the upper bound in Proposition \ref{localRademacherComplexityProp} is a sub-root function, which again, is no longer true if $\selfLipschitzExponent>1/2$. Nonetheless, it is natural to ask if the restriction $\selfLipschitzExponent \in [0,1/2]$ is truly necessary or purely an artefact of our proof. The following result shows that the range of $\selfLipschitzExponent$ cannot be extended.

\begin{theorem}\label{weCannotExtendTheRangeOfTheSelfBoundingLipschitzExponentThm} Given any $\selfLipschitzExponent>1/2$ the bound in Theorem \ref{optimisticRademacherBoundForMultioutputPrediction} is not true in general.
\end{theorem}
The proof of  Theorem \ref{weCannotExtendTheRangeOfTheSelfBoundingLipschitzExponentThm} is by contradiction. We consider a binary classification problem with  $\X = \{(x_r)_{r \in \N}:\sum_{r \in \N} x_r^2\leq 1\}$, $\Y = \{-1,+1\}$ and $\actionSpace=\R$, and investigate the bounded exponential loss $\lossFunction(u,y)= \min\{1,\exp(-u\cdot y)\}$. This loss is $(\selfLipschitzConstant,\selfLipschitzExponent)$-self-bounding Lipschitz for all $\selfLipschitzExponent \in [0,1]$. Hence, if the bound in Theorem \ref{optimisticRademacherBoundForMultioutputPrediction} were true for some $\selfLipschitzExponent>1/2$ then we could deduce a corresponding learning guarantee. It turns out that the resulting guarantee would exceed the minimax rate, if correct, so we deduce a contradiction and conclude that the bound cannot hold for $\selfLipschitzExponent>1/2$. For details see Appendix \ref{pfOfMinimaxRateSec}.

\section{An application to ensembles}\label{applicationToGBEnsemblesSec}

\newcommand{\numLeaves}{p}
\newcommand{\lOneRegConstant}{\tau}

In this section we consider an application which demonstrates how our learning guarantees may be applied to obtain tight risk bounds specific learning problems.  We shall consider ensembles of decision trees \cite{schapire2013boosting} which are an effective and widely used tool in applications \cite{chen2016xgboost}. Throughout this section we shall assume that $\X = \R^d$. We consider the function classes $\HClass_{\numLeaves,\lOneRegConstant} \subseteq \measurableMaps(\X,[-1,1]^{\numClasses})$ consisting of multi-output decision trees with $\ell_1$ regularised leaf nodes. More precisely, $\HClass_{\numLeaves,\lOneRegConstant}$ consists of all functions of the form $h(x) = ( w_{t(x), j})_{j \in [\numClasses]}$, where $t: \X \rightarrow [\numLeaves]$ is a decision tree with $\numLeaves$ leaves, where each internal node performs a binary split along a single feature, and $\bm{w}=(w_{l,j})_{(l,j) \in [\numLeaves]\times [\numClasses]} \in \R^{\numLeaves\numClasses}$ satisfies the $\ell_1$ constraint $\|w_{l\cdot}\|_1 = \sum_{j \in [\numClasses]}|w_{lj}| \leq  \lOneRegConstant$. We now give a bound for convex combinations of such decision trees.

\newcommand{\decisionTreeEnsembleComplexityTerm}{\mathfrak{C}_{n,\delta}(\alpha,\lOneRegConstant)}

\begin{theorem}\label{resultForEnsemblesOfDecisionTrees} Suppose we have $\hypothesisClassBound,\lossFunctionBound \geq 1$, $\selfLipschitzConstant>0$, $\selfLipschitzExponent \in [0,1/2]$ and a  $(\selfLipschitzConstant,\selfLipschitzExponent)$-{self-bounding Lipschitz} loss function $\lossFunction:\actionSpace \times \Y \rightarrow [0,\lossFunctionBound]$. Given $\delta \in (0,1)$, $n \in \N$ we define for each  $\alpha = (\alpha_t)_{t \in [T]}$, $\lOneRegConstant = (\lOneRegConstant_t)_{t \in [T]}$ $\in (0,\infty)^T$,
\begin{align*}
\decisionTreeEnsembleComplexityTerm:= &\left(\frac{\selfLipschitzConstant}{\sqrt{n}} \left(\sqrt{\numLeaves} \log^2(3n\numClasses d \hypothesisClassBound) \cdot \sum_{t \in [T]}\alpha_t\cdot \lOneRegConstant_t+1 \right)\right)^{\frac{1}{1-\selfLipschitzExponent}}+\frac{\lossFunctionBound}{n}\cdot (\log(1/\delta) +\log (\log n)).
\end{align*}

There exists a numerical constant $C_0$ such that given an i.i.d. sample $\sampleXY$ the following holds with probability at least $1-\delta$, for all ensembles $f= \sum_{t \in [T]}\alpha_t \cdot h_t$ where $\sum_{t \in [T]}\alpha_t \leq \hypothesisClassBound$ and $h_t \in \HClass_{\numLeaves,\lOneRegConstant_t}$,
\begin{align*}
\risk(f) \leq \empiricalRisk(f) + C_0 \cdot \left( \sqrt{\empiricalRisk(f) \cdot \decisionTreeEnsembleComplexityTerm} +\decisionTreeEnsembleComplexityTerm\right).
\end{align*}
\end{theorem}

Theorem \ref{resultForEnsemblesOfDecisionTrees} provides a unified guarantee for multi-output learning with ensembles of decision trees with $\ell_1$ trees. Before commenting upon the proof of Theorem \ref{resultForEnsemblesOfDecisionTrees} we shall highlight several important features:
\begin{itemize}[noitemsep,topsep=0pt]
\setlength\itemsep{0.5em}
    \item First and foremost, Theorem \ref{resultForEnsemblesOfDecisionTrees} gives guarantees for ensembles of decision trees with respect to a wide variety of losses including the \emph{multinomial logistic loss} for multi-class classification and the \emph{one versus all loss} for mulit-label classification, as well as implying margin based guarantees (see Section \ref{lossFuncExamplesSec}).
    \item Theorem \ref{resultForEnsemblesOfDecisionTrees} has a favourable dependency upon the number of examples whenever $\empiricalRisk(f)$ is sufficiently small, as is often the case for large ensembles of decision trees. For example, if we are using the multinomial logistic loss and $\empiricalRisk(f) \approx 0$, then Theorem \ref{resultForEnsemblesOfDecisionTrees} gives rise to a fast rate of $O(n^{-1})$.
    \item Theorem \ref{resultForEnsemblesOfDecisionTrees} has only logarithmic dependency upon the dimensionality of the output space $\numClasses$. This contrasts starkly with previous guarantees for multi-class learning with ensembles of decision trees \cite{kuznetsov2014multi,kuznetsov2015rademacher} which are linear with respect to the number of classes $\numClasses$.
\end{itemize}

The proof of Theorem \ref{resultForEnsemblesOfDecisionTrees} is a consequence of Theorem \ref{optimisticRademacherBoundForMultioutputPrediction} combined with the following lemma.

\begin{lemma}\label{rademacherComplexityOfL1NormConstrainedDecisionTrees} Given $n, \numClasses, d \in \N$, $\lOneRegConstant >0$ and $\numLeaves  \in \N \backslash\{1\}$,
\begin{align*}
\rademacherComplexity_{n \numClasses}\left(\Pi \circ \HClass_{\numLeaves,\lOneRegConstant}\right) \leq 2 \lOneRegConstant \sqrt{{\numLeaves\log(2\numClasses n d)}/{(n \numClasses)}}.     
\end{align*}
\end{lemma}
Lemma \ref{rademacherComplexityOfL1NormConstrainedDecisionTrees} follows from standard combinatorial arguments combined with Massart's lemma and a the convexity of Rademacher complexity. We can then deduce Theorem \ref{resultForEnsemblesOfDecisionTrees} by combining Theorem \ref{optimisticRademacherBoundForMultioutputPrediction} with Lemma \ref{rademacherComplexityOfL1NormConstrainedDecisionTrees} and applying standard results on the convexity of Rademacher complexity. For detailed proofs of both Theorem \ref{resultForEnsemblesOfDecisionTrees} and Lemma \ref{rademacherComplexityOfL1NormConstrainedDecisionTrees} see Appendix \ref{applicationToGBAppendix}.

In this section we have highlighted applications of our approach to ensembles of decision trees, yielding state of the art results. However, it is important to emphasize Theorem \ref{optimisticRademacherBoundForMultioutputPrediction} can be obtained to any multi-output prediction problem where one can obtain an upper bound on the Rademacher complexity $\rademacherComplexity_{n\numClasses}(\Pi \circ \F)$.

\section{Conclusions} %{Discussion and related work}
We presented a theoretical analysis of multi-output learning, based on a self-bounding Lipschitz condition. Under this condition, we obtained minimax-optimal rates with respect to both the sample size and the output dimension (up to logarithmic factors). We demonstrated an application to ensembles of decision trees, yielding state of the art guarantees. The main analytic tool was a new contraction inequality for the local Rademacher complexity of vector valued function classes with a self-bounding Lipschitz loss. Future work will exploit these results to develop further concrete applications of our framework.

\appendix
\onecolumn

%{\Large Supplementary material}

\section{The proof of Theorem \ref{optimisticRademacherBoundForMultioutputPrediction}}\label{proofOfMainOptimisticBoundThmSec}

To complete the proof of Theorem \ref{optimisticRademacherBoundForMultioutputPrediction} we combine Proposition \ref{localRademacherComplexityProp} with some results due to \cite{bousquet2002concentration}.

\newcommand{\localRadFunc}{\phi_n}
\newcommand{\localRadFixedPoint}{\hat{r}_n}

\begin{theorem}[\cite{bousquet2002concentration}]\label{bousquetNonNegLocalRademacherBound} Suppose we have a measurable space $\Z$ and a function class $\G \subseteq \measurableMaps(\Z,[0,\lossFunctionBound])$. For each $\zSequenceSizeN \in \Z^n$ and $g \in \G$ we let $\hat{\E}_{\zSequenceSizeN}(g) = n^{-1} \cdot \sum_{i \in [n]}g(z_i)$. Suppose we have a function $\localRadFunc:[0,\infty) \rightarrow [0,\infty)$ which is non-negative, non-decreasing, not identically zero, and $\localRadFunc(r)/\sqrt{r}$ is non-increasing. Suppose further that for all $\zSequenceSizeN \in \Z^n$ and $r>0$, 
\begin{align*}
\hat{\rademacherComplexity}_{\zSequenceSizeN}(\{g \in \G: \hat{\E}_{\zSequenceSizeN}(g)\leq r\}) \leq \localRadFunc(r).
\end{align*}
Let $\localRadFixedPoint$ be the largest solution of the equation $\localRadFunc(r)=r$. Suppose that $Z$ is a random variable with distribution $\probDistribution$ is a distribution on $\Z$ and let $\sample = \{Z_i\}_{i \in [n]} \in \Z^n$ be an i.i.d. which each $Z_i \sim \probDistribution$, an independent copy of $Z$. For any $\delta \in (0,1)$, the following holds with probability at least $1-\delta$, for all $g \in \G$,

\vspace{-6mm}
{{
\begin{align*}
\E(g) \leq \hat{\E}_{\sample}(g)+45\localRadFixedPoint + \sqrt{8 \localRadFixedPoint \E(g)}+\sqrt{4r_{0}\cdot \E(g)}+20r_{0},    
\end{align*}}}
\vspace{-6mm}

where $r_{0} = \lossFunctionBound\hspace{1mm}(\log(1/\delta)+6\log\log n)/n$.
\end{theorem}
\begin{proof}
This result is given in the penultimate line of the proof of ({Theorem 6.1}, \cite{bousquet2002concentration}).
\end{proof}
We also utilize the following lemma.
\begin{lemma}\label{rearrangingLemmaFromBousquet} Suppose that $t,B,C>0$ satisfies $t\leq B\sqrt{t}+C$. Then $t\leq B^2+C+B\sqrt{C}$.
\end{lemma}
\begin{proof}
See (Lemma 5.11, \cite{bousquet2002concentration})  with $t = x^2$.
\end{proof}
We can combine Theorem \ref{bousquetNonNegLocalRademacherBound} with Lemma \ref{rearrangingLemmaFromBousquet} to obtain the following bound.
\begin{corollary}\label{firstCorollaryFromBousquetNonNegLocalRademacherBound} Suppose that the assumptions of Theorem \ref{bousquetNonNegLocalRademacherBound} hold. For any $\delta \in (0,1)$, the following holds with probability at least $1-\delta$, for all $g \in \G$,
\begin{align*}
\E(g) \leq \hat{\E}_{\sample}(g)+90(\localRadFixedPoint+r_{0}) + 4\sqrt{ \hat{\E}_{\sample}(g)(\localRadFixedPoint+r_{0})}.
\end{align*}
\end{corollary}
\begin{proof} By Theorem \ref{bousquetNonNegLocalRademacherBound}, with probability at least $1-\delta$, for all $g \in \G$,
\begin{align*}
\E(g) &\leq \hat{\E}_{\sample}(g)+45\localRadFixedPoint + \sqrt{8 \localRadFixedPoint \E(g)}+\sqrt{4r_{0}\cdot \E(g)}+20r_{0}\\
 &\leq \hat{\E}_{\sample}(g)+45\localRadFixedPoint+20r_{0} + 4\sqrt{(\localRadFixedPoint+r_{0})\cdot\E(g)}.
\end{align*}
Applying Lemma \ref{rearrangingLemmaFromBousquet} with $B=4\sqrt{(\localRadFixedPoint+r_{0})}$ and $C = \hat{\E}_{\sample}(g)+45\localRadFixedPoint+20r_{0}$ we have
\begin{align*}
\E(g) &\leq 16(\localRadFixedPoint+r_{0})+(\hat{\E}_{\sample}(g)+45\localRadFixedPoint+20r_{0})\\&+4\sqrt{(\localRadFixedPoint+r_{0})(\hat{\E}_{\sample}(g)+45\localRadFixedPoint+20r_{0})}\\
&\leq  \hat{\E}_{\sample}(g)+90(\localRadFixedPoint+r_{0}) + 4\sqrt{ \hat{\E}_{\sample}(g)(\localRadFixedPoint+r_{0})},
\end{align*}
which proves the corollary.
\end{proof}
Both Theorem \ref{bousquetNonNegLocalRademacherBound} and Corollary \ref{firstCorollaryFromBousquetNonNegLocalRademacherBound} are uniform upper bounds in terms of the empirical risk. We can deduce a performance bound on the empirical risk minimizer by combining with Bernstein's inequality.

\begin{theorem}[\cite{bernstein1924modification}]\label{bernsteinsIneq} Let $W_i, \cdots, W_i \in [0,\lossFunctionBound]$ be bounded independent random variables with mean $\mu = \E[W_i]$. Then with probability at least $1-\delta$ we have
\begin{align*}
\frac{1}{n}\sum_{i \in [n]}W_i &\leq \mu + \sqrt{\frac{2\mu\lossFunctionBound \log(1/\delta)}{n}}+\frac{\lossFunctionBound \log(1/\delta)}{n}\\
&\leq 2\mu+\frac{3\lossFunctionBound \log(1/\delta)}{2n}.
\end{align*}
\end{theorem}
\begin{proof}
See Theorem 2.10 from \cite{boucheron2013concentration}.
\end{proof}

\begin{corollary}\label{secondCorollaryFromBousquetNonNegLocalRademacherBound} Suppose that the assumptions of Theorem \ref{bousquetNonNegLocalRademacherBound} hold and choose $g^* \in \argmin_{g \in \G}\{\E(g)\}$. Given $\zSequenceSizeN \in \Z^n$ we choose $\hat{g}_{\zSequenceSizeN} \in \argmin_{g \in \G}\{\hat{\E}_{\zSequenceSizeN}(g)\}$. For any $\delta \in (0,1)$, the following holds with probability at least $1-2\delta$
\begin{align*}
\E(\hat{g}_{\sample}) \leq \E(g^*)+9\sqrt{\E(g^*) \cdot \left(\localRadFixedPoint+r_0\right)}+100\left(\localRadFixedPoint+r_0\right).
\end{align*}
\end{corollary}
\begin{proof} By Corollary \ref{firstCorollaryFromBousquetNonNegLocalRademacherBound} the following holds with probability at least $1-\delta$ over $\sample$,
\begin{align*}
\E(\hat{g}_{\sample}) \leq \hat{\E}_{\sample}(\hat{g}_{\sample})+90(\localRadFixedPoint+r_{0}) + 4\sqrt{ \hat{\E}_{\sample}(\hat{g}_{\sample})(\localRadFixedPoint+r_{0})}.
\end{align*}
Morever, by the definition of $\hat{g}_{\sample}$ combined with Bernstein's inequality (Theorem \ref{bernsteinsIneq}), with probability at least $1-\delta$,
\begin{align*}
\hat{\E}_{\sample}(\hat{g}_{\sample})& \leq \hat{\E}_{\sample}({g}^*) \leq \E(g^*)+\sqrt{2\E(g^*)\cdot r_0}+r_0\\
& \leq 2\E(g^*)+3r_0.
\end{align*}
By the union bound we can combine the above two inequalities to show that with probability at least $1-2\delta$ we have,
\begin{align*}
\E(\hat{g}_{\sample}) &\leq \left(\E(g^*)+\sqrt{2\E(g^*)\cdot r_0}+r_0\right)+90(\localRadFixedPoint+r_{0})\\&\hspace{5mm}+ 4\sqrt{ \left(  2\E(g^*)+3r_0\right)\left(\localRadFixedPoint+r_{0}\right)}\\
& \leq \E(g^*)+9\sqrt{\E(g^*) \cdot \left(\localRadFixedPoint+r_0\right)}+100\left(\localRadFixedPoint+r_0\right).
\end{align*}
\end{proof}

We can now complete the proof of Theorem \ref{optimisticRademacherBoundForMultioutputPrediction}.

\begin{proof}[Proof of Theorem \ref{optimisticRademacherBoundForMultioutputPrediction}]
First let $\G = \lossFunction \circ \F = \{(x,y)\mapsto \lossFunction(f(x),y): f \in \F\}$. Note that for $g = \lossFunction \circ f$ with $f \in \F$ and $Z= (X,Y) \sim \probDistribution$ we have $\E_Z(g) = \risk(f,\probDistribution)$ and given $\zSequenceSizeN \in (\X\times \Y)^n$ we have $\hat{\E}_{\zSequenceSizeN}(g) = \empiricalRisk(f)$. Note also that under this correspondence $\lossFunction\circ\F|^r_{\xySequenceSizeN}=\{g \in \G: \hat{\E}_{\zSequenceSizeN}(g)\leq r\}$. Now define $\localRadFunc:[0,\infty) \rightarrow [0,\infty)$ by 
\begin{align*}
\localRadFunc(r)= \selfLipschitzConstant r^{\selfLipschitzExponent}  \left( 2^9 \sqrt{\numClasses} \cdot  \log^{3/2} \left(e\hypothesisClassBound n\numClasses\right)\cdot \rademacherComplexity_{n\numClasses}(\Pi \circ \F) +n^{-1/2}\right).
\end{align*}
Then $\localRadFunc$ is non-negative, non-decreasing and $\localRadFunc(r)/\sqrt{r}$ is non-increasing, since $\selfLipschitzExponent \in [0,1/2]$. Moreover, by Proposition \ref{localRademacherComplexityProp}, for each $\bm{z}  \in (\X\times \Y)^n$,
\begin{align*}
\hat{\rademacherComplexity}_{\zSequenceSizeN}\left(\{g \in \G: \hat{\E}_{\zSequenceSizeN}(g)\leq r\}\right) = \hat{\rademacherComplexity}_{\zSequenceSizeN}\left(\lossFunction\circ\F|^r_{\xySequenceSizeN}\right) \leq \localRadFunc(r).
\end{align*}
Note also that $\localRadFixedPoint:= \left(\selfLipschitzConstant  \left( 2^9 \sqrt{\numClasses} \cdot  \log^{3/2} \left(e\hypothesisClassBound n\numClasses\right)\cdot \rademacherComplexity_{n\numClasses}(\Pi \circ \F) +n^{-1/2}\right)\right)^{\frac{1}{1-\selfLipschitzExponent}}$ is the largest solution to $\localRadFunc(r)=r$. Hence, the two bounds in Theorem \ref{optimisticRademacherBoundForMultioutputPrediction} follow from Corollaries \ref{firstCorollaryFromBousquetNonNegLocalRademacherBound} and \ref{secondCorollaryFromBousquetNonNegLocalRademacherBound}, respectively. This completes the proof of Theorem \ref{optimisticRademacherBoundForMultioutputPrediction}.
\end{proof}

For completeness we also give a proof of Theorem \ref{additiveRademacherBoundForMultioutputPrediction}, which may be viewed as a mild generalization of Theorem 6 from \cite{lei2019data}. We use the following well known result.
\begin{theorem}[\cite{bartlett2002rademacher}]\label{standardRademacherBoundBartlett2002}  Suppose we have a measurable space $\Z$ and a function class $\G \subseteq \measurableMaps(\Z,[0,\lossFunctionBound])$. For each $\zSequenceSizeN \in \Z^n$ and $g \in \G$ we let $\hat{\E}_{\zSequenceSizeN}(g) = n^{-1} \cdot \sum_{i \in [n]}g(z_i)$. Suppose that $Z$ is a random variable with distribution $\probDistribution$ is a distribution on $\Z$ and let $\sample = \{Z_i\}_{i \in [n]} \in \Z^n$ be an i.i.d. which each $Z_i \sim \probDistribution$, an independent copy of $Z$. For any $\delta \in (0,1)$, the following holds with probability at least $1-\delta$, for all $g \in \G$,
\begin{align*}
\left|\E_Z(g)-\hat{\E}_{\zSequenceSizeN}(g)\right|& \leq 2 \E_{\sample}\left[ \hat{\rademacherComplexity}_{\sample}\left(\G\right)\right]+\sqrt{\frac{\log(2/\delta)}{2n}}.
\end{align*}
\end{theorem}

\begin{proof}[Proof of Theorem \ref{additiveRademacherBoundForMultioutputPrediction}] With the correspondence introduced in the proof of Theorem \ref{optimisticRademacherBoundForMultioutputPrediction},  Theorem \ref{standardRademacherBoundBartlett2002} implies that with probability at least $1-\delta$ over a sample $\sample = \{(X_i,Y_i)\}_{i \in [n]}$ with $(X_i,Y_i)\sim \probDistribution$ the following holds for all $f \in \F$,
\begin{align*}
\left|\risk(f,\probDistribution)-\empiricalRisk(f,\sample)\right|& \leq 2 \E_{\sample}\left[ \hat{\rademacherComplexity}_{\sample}\left(\lossFunction \circ \F\right)\right]+\sqrt{\frac{\log(2/\delta)}{2n}}.
\end{align*}
Hence, the result follows from Proposition \ref{localRademacherComplexityProp} by taking $r=\lossFunctionBound$ and $\selfLipschitzExponent=0$.
\end{proof}

\section{The self-bounding Lipschitz condition}\label{selfBoundingLipAppendix}
The proof of Lemma \ref{sufficiencyConditionForSelfBoundingLipschitz} starts with the following lemma.

\begin{lemma}\label{lemmaForProofOfSufficiencyConditionForSelfBoundingLipschitz} Suppose that $\varphi:\R \rightarrow [0,\infty)$ is a non-negative differentiable function satisfying:
\begin{enumerate}
\item The derivative $\varphi'(t)$ is non-negative on $[0,\infty)$;
\item $\forall t_0,t_1 >0$, $|\varphi'(t_1)-\varphi'(t_0)|\leq \left(\frac{\selfLipschitzConstant}{2}\right)^{\frac{1}{1-\selfLipschitzExponent}} \cdot |t_1-t_0|^{\frac{\selfLipschitzExponent}{1-\selfLipschitzExponent}}$.
\end{enumerate} 
Then $\forall t>0$, $\varphi'(t) \leq \selfLipschitzConstant \cdot \varphi(t)^{\selfLipschitzExponent}$. Moreover, for all $t>0$,  $\varphi(t)-\varphi(0)\leq \selfLipschitzConstant \cdot \varphi(t)^{\selfLipschitzExponent} \cdot t$.
\end{lemma}
\begin{proof} Fix $t>0$ and take $\Delta = 2 \selfLipschitzConstant^{-\frac{1}{\selfLipschitzExponent}}\cdot  \varphi'(t)^{\frac{1-\selfLipschitzExponent}{\selfLipschitzExponent}}$, which is positive by the first condition. By the non-negativity of $\varphi$ and the mean value theorem there exists some $s \in (t-\Delta,t)$
\begin{align*}
0 &\leq \varphi(t-\Delta) \leq \varphi(t)-\varphi'(s)\cdot \Delta  \\
& \leq \varphi(t) - \varphi'(t) \cdot \Delta + |\varphi'(s)-\varphi'(t)|\cdot \Delta\\
& \leq \varphi(t) - \varphi'(t) \cdot \Delta + \left(  (\selfLipschitzConstant/2)^{\frac{1}{1-\selfLipschitzExponent}} \cdot \Delta^{\frac{\selfLipschitzExponent}{1-\selfLipschitzExponent}}\right)\cdot \Delta\\
& \leq \varphi(t) - \varphi'(t) \cdot \Delta +  (\selfLipschitzConstant \cdot \Delta/2)^{\frac{1}{1-\selfLipschitzExponent}}\\
& \leq \varphi(t) - 2(\varphi'(t)/\selfLipschitzConstant)^{\frac{1}{\selfLipschitzExponent}}+ (\varphi'(t)/\selfLipschitzConstant)^{\frac{1}{\selfLipschitzExponent}}\\ &= \varphi(t) -(\varphi'(t)/\selfLipschitzConstant)^{\frac{1}{\selfLipschitzExponent}},
\end{align*}
where the fourth inequality follows from the second condition. Rearranging completes the proof of the first part of the lemma.

To prove the second part of the lemma we apply the mean value theorem combined with the first part of the lemma to obtain for some $s \in (0,t)$,
\begin{align*}
 \varphi(t)-\varphi(0)&= \varphi'(s) \cdot t \leq \left( \selfLipschitzConstant \cdot \varphi(s)^{\selfLipschitzExponent}\right) \cdot t\leq  \selfLipschitzConstant \cdot \varphi(t)^{\selfLipschitzExponent} \cdot t,
\end{align*}
where we used the non-negativity of $\varphi'$ on $[0,\infty)$ to ensure that $\varphi(s) \leq \varphi(t)$. This completes the proof of the lemma.
\end{proof}

\begin{proof}[Proof of Lemma \ref{sufficiencyConditionForSelfBoundingLipschitz}] Take $u,v \in \actionSpace$ and $y \in \Y$. Without loss of generality we assume that $\lossFunction(u,y) \leq \lossFunction(v,y)$ and let $\varphi_{u,y}$ be a function satisfying the conditions specified in the statement of the lemma. By combining the first two conditions with Lemma \ref{lemmaForProofOfSufficiencyConditionForSelfBoundingLipschitz} we see that $\varphi_{u,y}(t) - \varphi_{u,y}(0)   \leq  \selfLipschitzConstant \cdot \varphi_{u,y}(t)^{\selfLipschitzExponent}\cdot t$. Hence, by dividing through by $\varphi_{u,y}(t)^{\selfLipschitzExponent}$ and applying $\lossFunction(v,y) \leq \varphi_{u,y}(t)$ twice we have,
\begin{align*}
\lossFunction(v,y)^{1-\selfLipschitzExponent}- \selfLipschitzConstant \cdot t& \leq \varphi_{u,y}(t)^{1-\selfLipschitzExponent} - \selfLipschitzConstant \cdot t \\ &\leq \varphi_{u,y}(0)\cdot \varphi_{u,y}(t)^{-\selfLipschitzExponent} \\ &\leq \varphi_{u,y}(0)\cdot \lossFunction(v,y)^{-\selfLipschitzExponent}\\ &= \lossFunction(u,y)\cdot \lossFunction(v,y)^{-\selfLipschitzExponent}.
\end{align*}
Multiplying by $\lossFunction(v,y)^{\selfLipschitzExponent}$ and rearranging we have $\lossFunction(v,y)-\lossFunction(u,y) \leq \selfLipschitzConstant \cdot \lossFunction(v,y)^{\selfLipschitzExponent}$. Since $\lossFunction(u,y) \leq \lossFunction(v,y)$ this completes the proof of the lemma.
\end{proof}

\begin{proof}[Proof of Lemma \ref{cutOffOfSelfBoundingLipschitzLossIsAlsoSelfBoundingLipschitz}]  Take $u,v \in \actionSpace$ and $y \in \Y$. Without loss of generality we assume that $\tilde{\lossFunction}(u,y) \leq \tilde{\lossFunction}(v,y)$, so it suffices to show that 
\begin{align}\label{toProveInCutOffOfSelfBoundingLipschitzLossIsAlsoSelfBoundingLipschitz}\tilde{\lossFunction}(v,y)-\tilde{\lossFunction}(u,y) \leq \selfLipschitzConstant \cdot \tilde{\lossFunction}(v,y)^{\selfLipschitzExponent} \cdot \|u-v\|_{\infty}.
\end{align}
If $\lossFunction(u,y) \geq \lossFunctionBound$ then $\tilde{\lossFunction}(v,y)=\tilde{\lossFunction}(u,y)=\lossFunctionBound$, so (\ref{toProveInCutOffOfSelfBoundingLipschitzLossIsAlsoSelfBoundingLipschitz}) clearly holds. Thus, we can assume ${\lossFunction}(u,y)< \lossFunctionBound$, so  $\tilde{\lossFunction}(u,y)={\lossFunction}(u,y)$. By the $(\selfLipschitzConstant,\selfLipschitzExponent)$ self-bounding Lipschitz condition for $\lossFunction$ we have
\begin{align*}{\lossFunction}(v,y)-\tilde{\lossFunction}(u,y)&={\lossFunction}(v,y)-{\lossFunction}(u,y)\\ &\leq \selfLipschitzConstant \cdot {\lossFunction}(v,y)^{\selfLipschitzExponent} \cdot \|u-v\|_{\infty}.
\end{align*}
Equivalently, we have
\begin{align*}
{\lossFunction}(v,y)^{1-\selfLipschitzExponent}- \selfLipschitzConstant  \cdot \|u-v\|_{\infty} \leq \tilde{\lossFunction}(u,y) \cdot  {\lossFunction}(v,y)^{-\selfLipschitzExponent}.    
\end{align*}
Since $\tilde{\lossFunction}(v,y)\leq{\lossFunction}(v,y)$, we deduce
\begin{align*}
\tilde{\lossFunction}(v,y)^{1-\selfLipschitzExponent}- \selfLipschitzConstant  \cdot \|u-v\|_{\infty} \leq \tilde{\lossFunction}(u,y) \cdot  \tilde{\lossFunction}(v,y)^{-\selfLipschitzExponent}.    
\end{align*}
Rearranging gives (\ref{toProveInCutOffOfSelfBoundingLipschitzLossIsAlsoSelfBoundingLipschitz}) and completes the proof of the lemma.
\end{proof}

The following result shows an example application of Lemma \ref{sufficiencyConditionForSelfBoundingLipschitz}. We may verify the self-bounding Lipschitz condition for other loss functions in a similar manner. 

\begin{prop}\label{checkingTheSelfBoundingLipschitzConditionForMultinomialLogisticLoss}  Take $\Y = [\numClasses]$ and define the multinomial logistic loss $\lossFunction:\actionSpace \times \Y \rightarrow [0,\infty)$ is defined by \[\lossFunction(u,y) = \log (\sum_{j \in [\numClasses]} \exp(u_j-u_y)),\]
where $u = (u_j)_{j \in [\numClasses]}$ and $y \in [\numClasses]$. It follows that $\lossFunction$ is  $(\selfLipschitzConstant,\selfLipschitzExponent)$-\emph{self-bounding Lipschitz} with $\selfLipschitzConstant = 1$ and $\selfLipschitzExponent = 1/2$.
\end{prop}
The proof of Proposition \ref{checkingTheSelfBoundingLipschitzConditionForMultinomialLogisticLoss} requires the following elementary lemma.
\begin{lemma}\label{lemmaForCheckingTheSelfBoundingLipschitzConditionForMultinomialLogisticLoss} Given any $A >0$ the function $\varphi:\R \rightarrow (0,\infty)$ defined by $\varphi(t) = \log(1+A \cdot \exp(2t))$ is differentiable $\varphi'(t_0)> 0$ and and  $|\varphi'(t_0)-\varphi'(t_1)|\leq  |t_1-t_0|$ for all $t_0,t_1 \in \R$.
\end{lemma}
\begin{proof}
We begin by computing the first three derivatives,
\begin{align*}
\varphi'(t)& = \frac{2A \cdot \exp(2t)}{1+A \cdot \exp(2t)}\\
\varphi''(t)& = \frac{4A \cdot \exp(2t)}{\left(1+A \cdot \exp(2t)\right)^2}\\
\varphi'''(t)& = \frac{8A \cdot \exp(2t)}{\left(1+A \cdot \exp(2t)\right)^3}\cdot \left(1-A \cdot \exp(2t)\right).
\end{align*}
Clearly we have $\varphi'(t),\varphi''(t)> 0$ for all $t \in \R$. Moreover, by inspecting the third derivative we see that $\varphi''$ has a unique maximum where $A\cdot \exp(2t)=1$. This implies that $\varphi$ is twice differentiable with  $|\varphi''(t)| \leq 1/4$ for all $t \in \R$. By the mean value theorem this yields $|\varphi'(t_0)-\varphi'(t_1)|\leq |t_1-t_0|$ for all $t_0,t_1 \in \R$.
\end{proof}
\begin{proof}[Proof of Proposition \ref{checkingTheSelfBoundingLipschitzConditionForMultinomialLogisticLoss}] To complete the proof we $A_{u,y}:=\sum_{j \in [\numClasses]\backslash \{y\}} \exp(u_j-u_y)$ and define $\varphi_{u,y}(t):= \log\left(1+A_{u,y}\cdot \exp(2t)\right)$. We can apply Lemma \ref{lemmaForCheckingTheSelfBoundingLipschitzConditionForMultinomialLogisticLoss} to confirm that $\varphi_{u,y}$ satisfies the conditions of Lemma \ref{sufficiencyConditionForSelfBoundingLipschitz}. Hence, the conclusion of Proposition \ref{checkingTheSelfBoundingLipschitzConditionForMultinomialLogisticLoss} follows from Lemma \ref{sufficiencyConditionForSelfBoundingLipschitz}.
\end{proof}

\section{Minimax lower bounds}\label{pfOfMinimaxRateSec}

In this section we present the proofs for Theorems \ref{minimaxOptimalityRelizableProblemsHilbertSpaceThm} and \ref{weCannotExtendTheRangeOfTheSelfBoundingLipschitzExponentThm}. 

\subsection{Proof of Theorem \ref{minimaxOptimalityRelizableProblemsHilbertSpaceThm}}

We begin with Theorem \ref{minimaxOptimalityRelizableProblemsHilbertSpaceThm}, the proof of which consists of an upper bound (Proposition \ref{minimaxOptimalityRelizableProblemsHilbertSpaceUpperBoundProp}) and a lower bound (Proposition \ref{minimaxOptimalityRelizableProblemsHilbertSpaceLowerBoundProp}). Let $\X$ be an infinite set and take $\actionSpace=\Y = [-1,1]^{\numClasses}$.

\newcommand{\multiOutputomplexityTermNoDelta}{\Gamma_{n,\numClasses,1/n}^{\selfLipschitzConstant,\selfLipschitzExponent}(\F)}
\begin{prop}\label{minimaxOptimalityRelizableProblemsHilbertSpaceUpperBoundProp} Suppose that $\hat{\phi}$ is the \emph{empirical risk minimization} algorithm which takes as input a sample $\sample = \{(X_i,Y_i)\}_{i \in [n]}$, and for a given function class $\F \subseteq \measurableMaps(\X,\actionSpace)$ and a loss function $\lossFunction$, outputs an empirical risk minimizer $\hat{\classifier}_{\sample} \in  \argmin_{f \in \F}\{\empiricalRisk(f,\sample)\}$. Suppose that $\selfLipschitzConstant\geq 1$, $\selfLipschitzExponent \in [0,1/2]$, $n$, $\numClasses \in \N$ and $\kappa \geq 1$. Then
given any  $(\selfLipschitzConstant,\selfLipschitzExponent)$-self-bounding Lipschitz loss function $\lossFunction:\actionSpace\times \Y \rightarrow [0,1]$, any function class $\F \subseteq \measurableMaps(\X,\actionSpace)$ with  $\rademacherComplexity_{n\numClasses}(\Pi \circ \F) \leq \sqrt{\kappa/(nq)}$ and any $(\lossFunction,\F)$-realizable problem $\probDistribution$,
\begin{align*}
\E_{\sample}\left[   \risk\left(\hat{\phi}_{\sample},\probDistribution\right)\right] \leq (8C_1+1) \cdot \log^{3} \left(e n\numClasses\right) \cdot \left(\selfLipschitzConstant \cdot \sqrt{ \frac{\kappa}{n}}\right)^{\frac{1}{1-\selfLipschitzExponent}}.
\end{align*}
\end{prop}
\begin{proof} Since $\T = (\lossFunction,\probDistribution,\F)$ is a realizable problem, there exists $f^* \in \F$ with $\risk(f^*,\probDistribution)=0$. By applying the second part of Theorem \ref{optimisticRademacherBoundForMultioutputPrediction} we see that with probability at least $1-n^{-1}$ of $\sample$,
\begin{align*}
\risk\left(\hat{\phi}_{\sample},\probDistribution\right) \leq C_1 \cdot \multiOutputomplexityTermNoDelta,    
\end{align*}
where
\begin{align*}
\multiOutputomplexityTermNoDelta&= \left(\selfLipschitzConstant \left( \sqrt{\numClasses} \cdot  \log^{3/2} \left(e n\numClasses\right)\cdot \rademacherComplexity_{n\numClasses}(\Pi \circ \F) + \frac{1}{\sqrt{n}}\right)\right)^{\frac{1}{1-\selfLipschitzExponent}}+\frac{1}{n}\cdot (\log(n) +\log (\log n))\\
& \leq \left(\selfLipschitzConstant \left( \log^{3/2} \left(e n\numClasses\right)\cdot \sqrt{ \frac{\kappa}{n}} + \frac{1}{\sqrt{n}}\right)\right)^{\frac{1}{1-\selfLipschitzExponent}}+\frac{2 \log(n)}{n}\\
& \leq 4\log^{3} \left(e n\numClasses\right) \cdot \left(\selfLipschitzConstant \cdot \sqrt{ \frac{\kappa}{n}}\right)^{\frac{1}{1-\selfLipschitzExponent}}+\frac{2 \log(n)}{n} \leq 8\log^{3} \left(e n\numClasses\right) \cdot \left(\selfLipschitzConstant \cdot \sqrt{ \frac{\kappa}{n}}\right)^{\frac{1}{1-\selfLipschitzExponent}},
\end{align*}
where the final inequality uses the fact that $\selfLipschitzConstant$, $\kappa\geq 1$ and $\selfLipschitzExponent \leq 1/2$ so $1/(1-\selfLipschitzExponent) \leq 2$. Hence, given that the non-negative loss function $\lossFunction$ is bounded above by $1$ we can take expectations and obtain,
\begin{align*}
\E\left[\risk\left(\hat{\phi}_{\sample},\probDistribution\right)\right] \leq C_1 \left( 8\log^{3} \left(e n\numClasses\right) \cdot \left(\selfLipschitzConstant \cdot \sqrt{ \frac{\kappa}{n}}\right)^{\frac{1}{1-\selfLipschitzExponent}}\right)+\frac{1}{n} \leq (8C_1+1) \cdot \log^{3} \left(e n\numClasses\right) \cdot \left(\selfLipschitzConstant \cdot \sqrt{ \frac{\kappa}{n}}\right)^{\frac{1}{1-\selfLipschitzExponent}}.
\end{align*}
This completes the proof of the proposition.
\end{proof}
The lower bound is more interesting as it requires constructing a family of examples where no strategy does well.

\begin{prop}\label{minimaxOptimalityRelizableProblemsHilbertSpaceLowerBoundProp} Given any $\selfLipschitzConstant\geq 1$, $\selfLipschitzExponent \in [0,1/2]$, $n$, $\numClasses \in \N$, $\kappa \leq  {n}/\selfLipschitzConstant^{2}$ and $\actionSpace=\Y = [-1,1]^{\numClasses}$, there exists a $(\selfLipschitzConstant,\selfLipschitzExponent)$-self-bounding Lipschitz loss function $\lossFunction:\actionSpace \times \Y \rightarrow [0,1]$ and a function class $\F \subseteq \measurableMaps(\X,\actionSpace)$ with $\rademacherComplexity_{n \numClasses}(\Pi \circ \F) \leq \sqrt{\kappa/(n\numClasses)}$ such that for any algorithm $\hat{\phi}$, which takes as input a sample $\sample =  \{(X_i,Y_i)\}_{i \in [n]} \in (\X\times \Y)^n$ and outputs a function $\hat{\phi}_{\sample} \in \F$, there exists a $(\lossFunction,\F)$-realizable problem $\probDistribution$ with the following lower bound
\begin{align*}
\E_{\sample}\left[   \risk\left(\hat{\phi}_{\sample},\probDistribution\right)\right] \geq  2^{-8}\cdot \left(\selfLipschitzConstant \cdot \sqrt{ \frac{\kappa}{n}}\right)^{\frac{1}{1-\selfLipschitzExponent}}.
\end{align*}

\end{prop}
For the purpose of the proof we utilize a variant of the loss function considered in Example \ref{multiOutputRegressionLosses}. For each  $\selfLipschitzConstant\geq 1$, $\selfLipschitzExponent \in [0,1/2]$ we define a loss function $\lossFunction_{\selfLipschitzConstant,\selfLipschitzExponent}:\actionSpace \times \Y \rightarrow [0,1]$ by
\begin{align*}
\lossFunction_{\selfLipschitzConstant,\selfLipschitzExponent}(u,y) := \min\left\lbrace \frac{1}{2^5}\cdot \left( \selfLipschitzConstant \cdot \|u-y\|_{\infty}\right)^{\frac{1}{1-\selfLipschitzExponent}}, 1 \right\rbrace.
\end{align*}

\begin{lemma}\label{verifyingSelfBoundingLipschitzPropLemmaForMultiOutputRegLossLemma} Given any  $\selfLipschitzConstant\geq 1$, $\selfLipschitzExponent \in [0,1/2]$ the loss function $\lossFunction_{\selfLipschitzConstant,\selfLipschitzExponent}:\actionSpace \times \Y \rightarrow [0,1]$ is $(\selfLipschitzConstant,\selfLipschitzExponent)$-self-bounding Lipschitz.
\end{lemma}
Before proving Lemma \ref{verifyingSelfBoundingLipschitzPropLemmaForMultiOutputRegLossLemma} let's recall a couple of standard lemmas.
\begin{lemma}\label{concavityLemma} Given any $a$, $b>0$ and $\gamma \in [0,1]$ we have $(a+b)^{\gamma}\leq a^{\gamma}+b^{\gamma} \leq 2^{1-\gamma}\cdot (a+b)^{\gamma}$.
\end{lemma}
\begin{proof} Since $\gamma \leq 1$, $z \mapsto z^{\gamma}$ is concave, so by Jensen's inequality we have
\begin{align*}
\frac{a}{a+b}\cdot (a+b)^{\gamma} &= \frac{a}{a+b}\cdot (a+b)^{\gamma} +\frac{b}{a+b} \cdot 0^{\gamma}\leq \left( \frac{a}{a+b}\cdot (a+b) +\frac{b}{a+b} \cdot 0\right)^{\gamma}  = a^{\gamma}.
\end{align*}
Similarly, $\frac{b}{a+b}\cdot (a+b)^{\gamma} \leq b^{\gamma}$. Summing up these two inequalities yields the lower bound. The upper bound also follows from Jensen's inequality.
\end{proof}

\begin{lemma}\label{standardHolderContLemma} Given $\gamma \in [0,1]$ the function $g:\R \rightarrow \R$ defined by $g(z) = \text{sign}(z) \cdot |z|^{\gamma}$ satisfies $|g(z_1)-g(z_0)| \leq 2^{1-\gamma}\cdot |z_1-z_0|^{\gamma}$ for all $z_0$, $z_1$ $\in \R$.
\end{lemma}
\begin{proof} Take $z_0,z_1 \in \R$. Without loss of generality we may assume that $|z_1|\geq |z_0|$. If $\text{sign}(z_1)=\text{sign}(z_0)$ then
\begin{align*}
\left|g(z_1)-g(z_0)\right|  =  |z_1|^{\gamma}-|z_0|^{\gamma} \leq \left( |z_1|-|z_0|\right)^{\gamma} = \left|z_1-z_0\right|^{\gamma},
\end{align*}
where the inequality follows from Lemma \ref{concavityLemma} by taking $a=|z_0|$ and $b=|z_1|-|z_0|$. On the other hand, if  $\text{sign}(z_1)\neq \text{sign}(z_0)$ then
\begin{align*}
\left|g(z_1)-g(z_0)\right|  =  |z_1|^{\gamma}+|z_0|^{\gamma} \leq 2^{1-\gamma} \cdot \left(|z_1|+|z_0|\right)^{\gamma} = 2^{1-\gamma} \cdot\left|z_1-z_0\right|^{\gamma}.
\end{align*}

\end{proof}

\begin{proof}[Proof of Lemma \ref{verifyingSelfBoundingLipschitzPropLemmaForMultiOutputRegLossLemma}] By Lemma \ref{cutOffOfSelfBoundingLipschitzLossIsAlsoSelfBoundingLipschitz} it suffices to verify that the loss function  $\lossFunction_{\selfLipschitzConstant,\selfLipschitzExponent}^{\sharp}:\actionSpace \times \Y \rightarrow [0,\infty)$, defined by 
\begin{align*}
\lossFunction_{\selfLipschitzConstant,\selfLipschitzExponent}^{\sharp}(u,y) :=  \frac{1}{2^5}\cdot \left( \selfLipschitzConstant \cdot \|u-y\|_{\infty}\right)^{\frac{1}{1-\selfLipschitzExponent}}
\end{align*}
is $(\selfLipschitzConstant,\selfLipschitzExponent)$-self-bounding Lipschitz. For the special case of $\selfLipschitzExponent=0$ this follows straightforwardly from the definitions. Hence, we may assume without loss of generality that $\selfLipschitzExponent \in (0,1/2]$. To demonstrate that $\lossFunction_{\selfLipschitzConstant,\selfLipschitzExponent}^{\sharp}$ is $(\selfLipschitzConstant,\selfLipschitzExponent)$-self-bounding Lipschitz for $\selfLipschitzExponent \in (0,1/2]$ we apply Lemma \ref{sufficiencyConditionForSelfBoundingLipschitz} with $\varphi_{u,y}:\R \rightarrow [0,\infty)$ defined by 
\begin{align*}
\varphi_{u,y}(t):=  \frac{1}{2^5}\cdot \left( \selfLipschitzConstant \cdot \left| \|u-y\|_{\infty}+t\right|\right)^{\frac{1}{1-\selfLipschitzExponent}}.    
\end{align*}
We now check properties 1 - 4 for Lemma \ref{sufficiencyConditionForSelfBoundingLipschitz}. Property 1 is immediate from the definition. Property 2 follows straightforwardly from the triangle inequality. By computing the derivative we have
\begin{align*}
\varphi_{u,y}'(t)=  \frac{\selfLipschitzConstant^{\frac{1}{1-\selfLipschitzExponent}}}{2^5(1-\selfLipschitzExponent)}\cdot \text{sign}\left(\|u-y\|_{\infty}+t\right)\cdot  \left|  \|u-y\|_{\infty}+t\right|^{\frac{\selfLipschitzExponent}{1-\selfLipschitzExponent}} 
\end{align*}
Hence $\varphi_{u,y}$ is differentiable with non-negative derivative on $[0,\infty)$, so property 3 holds. Finally, by Lemma \ref{standardHolderContLemma} we have that for all $t_0$, $t_1 \in \R$
\begin{align*}
\left|\varphi_{u,y}'(t_1)-\varphi_{u,y}'(t_0)\right| \leq \frac{ (2\selfLipschitzConstant)^{\frac{1}{1-\selfLipschitzExponent}}}{2^5(1-\selfLipschitzExponent)}\cdot |t_1-t_0|^{\frac{\selfLipschitzExponent}{1-\selfLipschitzExponent}}\leq \left(\frac{ \selfLipschitzConstant}{2}\right)^{\frac{1}{1-\selfLipschitzExponent}}\cdot |t_1-t_0|^{\frac{\selfLipschitzExponent}{1-\selfLipschitzExponent}},
\end{align*}
which confirms Property 4. Thus, we may apply Lemma \ref{sufficiencyConditionForSelfBoundingLipschitz} to show that $\lossFunction_{\selfLipschitzConstant,\selfLipschitzExponent}^{\sharp}$, and hence $\lossFunction_{\selfLipschitzConstant,\selfLipschitzExponent}$ is $(\selfLipschitzConstant,\selfLipschitzExponent)$-self-bounding Lipschitz.

\end{proof}
We also introduce a function class $\F$ defined as follows. First let $\ell_2$ denote the canonical Hilbert space constructed by $\ell_2 = \{(a_r)_{r\in \N}:\sum_{r \in \N}(a_r)^2 <\infty\}$ with the standard inner product $\left\langle a, b \right\rangle = \sum_{r\in \N}a_r \cdot b_r$ for $a=(a_r)_{r \in \N}$, $b=(b_r)_{r \in \N} \in \ell_2$, and $\|\cdot \|_2$ the corresponding norm. Let $\mathbb{H}_{\numClasses} = (\ell_2)^{\numClasses} = \{ (a_r^j)_{j \in [\numClasses], r \in \N}: \sum_{r \in \N,j \in [\numClasses]}(a_r^j)^2<\infty\}$, with the inner product $\left\langle a,b \right\rangle_{\mathbb{H}_{\numClasses}} = \sum_{j \in [\numClasses]} \left\langle a^j , b^j \right\rangle$ for $a=(a^j)_{j \in [\numClasses]}$, $b=(b^j)_{j \in [\numClasses]} \in \mathbb{H}_{\numClasses}$ and $\| \cdot \|_{\mathbb{H}_{\numClasses}}$ the corresponding Hilbert space norm. We also let $\|\cdot \|_{\infty}$ be the norm on $\mathbb{H}_{\numClasses}$ defined by $\|a\|_{\infty}=\sup_{j \in [\numClasses], r \in \N}\{|a^j_r|\}$ for $a=(a^j_r)_{j \in [\numClasses],r \in \N} \in \mathbb{H}_{\numClasses}$. Note that whilst $\mathbb{H}_{\numClasses}$ is isomorphic to $\ell_2$ it is useful in this instance to view $\mathbb{H}_{\numClasses}$ as a distinct space. For each $t \in \N$ we let $e(t)=(e(t)_r)_{r \in \N} \in \ell_2$ denote the $r$-th canonical basis element where $e(t)_t=1$ and $e(t)_r=0$ for $r \in \N\backslash \{t\}$. In addition, we let $\{e(t,k)\}_{t \in \N, k \in [\numClasses]}$ be the canonical basis for $\mathbb{H}_{\numClasses}$ defined by $e(t,k) = (e(t,k)^{j})_{j \in [\numClasses]} \in \mathbb{H}_{\numClasses}$ where $e(t,k)^k = e(t) \in \ell_2$ and $e(t,k)^j = \bf{0} \in \ell_2$ for $j \in [\numClasses]\backslash \{k\}$. Let $\omega:\X \rightarrow \N$ be any surjective map, which must exist as $\X$ has infinite cardinality. Given $\kappa>0$ we let $ \mathcal{W}_{\kappa}:= \{ w \in \mathbb{H}_{\numClasses}\text{ with } \|w\|_{\infty} \leq 1\text{ and } \|w\|_{\mathbb{H}_{\kappa}}\leq \sqrt{\kappa}\}$. For each $w=(w^j)_{j \in [\numClasses]} \in \mathcal{W}_{\kappa}$, we define $f_{w}:\X\rightarrow \actionSpace$ by  $f_w(x) = \left( \left\langle w^j,e({\omega(x)}) \right\rangle  \right)_{j \in [\numClasses]} \in \actionSpace$. Finally, we let $\F:= \left\lbrace f_w: w \in \mathcal{W}_{\kappa}\right\rbrace$.

\begin{lemma}\label{rademacherComplexityBoundFKappa} We have $\rademacherComplexity_{n \numClasses}(\Pi \circ \F) \leq \sqrt{\kappa/(n\numClasses)}$.
\end{lemma}
\begin{proof} It suffices to show that for all $z=(z_s)_{s \in [n \numClasses]}$ with $z_s = (x_s,j_s)\in \X \times [\numClasses]$, we have $\hat{\rademacherComplexity}_z(\Pi \circ \F) \leq \sqrt{\kappa/(n\numClasses)}$. For each $s \in [n \numClasses]$ take $t_s\in \N$ so that $t_s=\omega(x_s)$. We then have
\begin{align*}
\rademacherComplexity_{n \numClasses}(\Pi \circ \F) &= \sup_{\tilde{\G} \subseteq \Pi \circ \F: |\tilde{\G}|<\infty}  \E_{\sigmaSequenceSizeN}\left( \sup_{g \in \tilde{\G}} \frac{1}{n \numClasses} \sum_{s \in [n \numClasses]} \sigma_s \cdot g(z_s)\right)\\ &= \sup_{\tilde{\F} \subseteq  \F: |\tilde{\F}|<\infty}  \E_{\sigmaSequenceSizeN}\left( \sup_{g \in \tilde{\G}} \frac{1}{n \numClasses} \sum_{s \in [n \numClasses]} \sigma_s \cdot  (\Pi \circ f)(x_s,j_s)\right)\\
&=  \sup_{\tilde{\mathcal{W}}_{\kappa} \subseteq  \mathcal{W}_{\kappa}: |\tilde{\mathcal{W}}_{\kappa} |<\infty}  \E_{\sigmaSequenceSizeN}\left( \sup_{w \in \tilde{\mathcal{W}}_{\kappa}} \frac{1}{n \numClasses} \sum_{s \in [n \numClasses]} \sigma_s \cdot  \pi_{j_s}\left( f_w(x_s)\right)\right)\\
&=  \sup_{\tilde{\mathcal{W}}_{\kappa} \subseteq  \mathcal{W}_{\kappa}: |\tilde{\mathcal{W}}_{\kappa} |<\infty}  \E_{\sigmaSequenceSizeN}\left( \sup_{w \in \tilde{\mathcal{W}}_{\kappa}} \frac{1}{n \numClasses} \sum_{s \in [n \numClasses]} \sigma_s \cdot  \left\langle w^{j_s},e(\omega(x_s)) \right\rangle \right)\\
&=  \sup_{\tilde{\mathcal{W}}_{\kappa} \subseteq  \mathcal{W}_{\kappa}: |\tilde{\mathcal{W}}_{\kappa} |<\infty}  \E_{\sigmaSequenceSizeN}\left( \sup_{w \in \tilde{\mathcal{W}}_{\kappa}} \frac{1}{n \numClasses} \sum_{s \in [n \numClasses]} \sigma_s \cdot  \left\langle w^{j_s},e(t_s) \right\rangle \right)\\
&=  \sup_{\tilde{\mathcal{W}}_{\kappa} \subseteq  \mathcal{W}_{\kappa}: |\tilde{\mathcal{W}}_{\kappa} |<\infty}  \E_{\sigmaSequenceSizeN}\left( \sup_{w \in \tilde{\mathcal{W}}_{\kappa}} \frac{1}{n \numClasses} \sum_{s \in [n \numClasses]} \sigma_s \cdot  \left\langle w,e(t_s,j_s) \right\rangle_{\mathbb{H}_{\kappa}} \right)\\
&=  \sup_{\tilde{\mathcal{W}}_{\kappa} \subseteq  \mathcal{W}_{\kappa}: |\tilde{\mathcal{W}}_{\kappa} |<\infty}  \E_{\sigmaSequenceSizeN}\left( \sup_{w \in \tilde{\mathcal{W}}_{\kappa}} \frac{1}{n \numClasses}  \left\langle w,\sum_{s \in [n \numClasses]} \sigma_s \cdot e(t_s,j_s) \right\rangle_{\mathbb{H}_{\kappa}} \right)\\
&\leq  \sup_{\tilde{\mathcal{W}}_{\kappa} \subseteq  \mathcal{W}_{\kappa}: |\tilde{\mathcal{W}}_{\kappa} |<\infty}  \E_{\sigmaSequenceSizeN}\left( \sup_{w \in \tilde{\mathcal{W}}_{\kappa}} \frac{1}{n \numClasses}  \|w\|_{\mathbb{H}_{\kappa}} \cdot \left\|\sum_{s \in [n \numClasses]} \sigma_s \cdot e(t_s,j_s) \right\|_{\mathbb{H}_{\kappa}} \right)\\
&\leq  \frac{{\kappa}^{\frac{1}{2}}}{n \numClasses} \cdot \E_{\sigmaSequenceSizeN}\left(  \left\|\sum_{s \in [n \numClasses]} \sigma_s \cdot e(t_s,j_s) \right\|_{\mathbb{H}_{\kappa}} \right)\\
&\leq  \frac{{\kappa}^{\frac{1}{2}}}{n \numClasses} \cdot \left(\E_{\sigmaSequenceSizeN}\left(  \left\|\sum_{s \in [n \numClasses]} \sigma_s \cdot e(t_s,j_s) \right\|_{\mathbb{H}_{\kappa}}^2 \right)\right)^{\frac{1}{2}} \leq \sqrt{ \frac{\kappa}{n\numClasses}},
\end{align*}
where the penultimate inequality follows from Jensen's inequality.
\end{proof}

We now take $\Sigma:= \{-1,+1\}^{2n}$ and define a family of distributions $\{\probDistribution(\sigma)\}_{\sigma \in \Sigma}$ as follows. First, for each $r \in [2n]$ we choose $x^{(r)} \in \omega^{-1}(r) \subseteq \X$ and let $\mu$ be the uniform measure on the set $\{x^{(r)}\}_{r \in [2n]}$, so $\mu(\{x^{(r)}\})=1/(2n)$ for $r \in [2n]$. We then fix $\Delta= \sqrt{\kappa/(2n)}$. The choice of $\Delta$ will be explained shortly. For each $\sigma = (\sigma_r)_{r \in [2n]} \in \Sigma$ we define $w(\sigma)=(w(\sigma)^j_r)_{j \in [\numClasses],r \in \N} \in \mathbb{H}_{\numClasses}$ by 
\begin{align*}
w(\sigma)^j_r= \begin{cases}\Delta \cdot \sigma_r &\text{ for }j=1\text{ and }r \in [2n]\\    
0 &\text{ otherwise. }.
\end{cases}
\end{align*}
The choice of $\Delta$ is maximal so that for all $\sigma \in \Sigma$, $\|w(\sigma)\|_{\infty}\leq 1$ and $\|w(\sigma)\|_{\mathbb{H}_{\numClasses}} \leq \sqrt{\kappa}$, which ensures that $f_{w(\sigma)} \in \F$. Indeed, since $\kappa \leq n \cdot \selfLipschitzConstant^{-2}$ and $\selfLipschitzConstant \geq 1$, we have $\Delta \in [0,1]$ and so $\|w(\sigma)\|_{\infty} \leq 1$. Moreover, $\|w(\sigma)\|_{\mathbb{H}_{\numClasses}}  = \sqrt{ \sum_{r \in [2n]} \Delta^2 }=\sqrt{\kappa}$. Thus,  for all $\sigma \in \Sigma$, $f_{w(\sigma)} \in \F$. Finally, for each $\sigma \in \Sigma$ we let $P(\sigma)$ be the unique distribution on $\X \times \Y$ such that $P(\sigma)_X=\mu$ is the marginal distribution over $\X$ and for each $x \in \X$, the conditional distribution of $Y$ given $X$ is $P(\sigma)_{Y|x}$ is concentrated on the single point $f_{w(\sigma)}(x)$.

\begin{lemma}\label{realizableProblemsLowerBoundLemma} For all $\sigma \in \Sigma$, the probability distribution $\probDistribution(\sigma)$ is a $(\lossFunction,\F)$-realizable problem.
\end{lemma}
\begin{proof} It suffices to show that for each $\sigma \in \Sigma$, $\risk\left(f_{w(\sigma)},\probDistribution(\sigma)\right)=0$, since $f_{w(\sigma)} \in \F$. Moreover, this follows from the fact that for each $x \in \X$, the conditional distribution of $Y$ given $X$ is $P(\sigma)_{Y|x}$ is concentrated on the single point $f_{w(\sigma)}(x)$ and by construction $\lossFunction(y,y)=0$ for all $y \in \R$, so 
\begin{align*}
\risk\left(f_{w(\sigma)},\probDistribution(\sigma)\right) &= \E_{(X,Y) \sim \probDistribution(\sigma)}\left( \lossFunction(f_{w(\sigma)}(X),Y)\right) = \E_{X \sim \mu}\left( \lossFunction\left(f_{w(\sigma)}(X),f_{w(\sigma)}(X)\right)\right)=0. 
\end{align*}
\end{proof}

We now show that no mapping $\phi$ can do well on a large set of possible distributions. 
\begin{lemma}\label{difficultLemmaInRealizableMinimaxLB} Take $A \subseteq [2n]$ and choose $\{\sigma_r\}_{r \in A} \in \{-1,+1\}^A$. Given any mapping $\phi \in \measurableMaps(\X,\actionSpace)$ we have
\begin{align*}
\E_{\{\sigma_r\}_{r \in [2n]\backslash A}}\left[\risk\left(\phi,\probDistribution(\sigma)\right)\right] \geq  \frac{2n - |A|}{2^7 \cdot n}\cdot (\selfLipschitzConstant \cdot \Delta)^{ \frac{1}{1-\selfLipschitzExponent}}.
\end{align*}
where $\{\sigma_r\}_{r \in [2n]\backslash A}$ is sampled from the uniform distribution on $\{-1,+1\}^{[2n]\backslash A}$  and  $\sigma=\{\sigma_r\}_{r \in [2n]} \in \Sigma$.
\end{lemma}

\begin{proof}
Observe that for each  $\sigma=\{\sigma_s\}_{s \in [2n]} \in \Sigma$,
\begin{align*}
\risk\left(\phi,\probDistribution(\sigma)\right) &= \E_{(X,Y) \sim \probDistribution(\sigma)}\left[\lossFunction(\phi(X),Y)\right] \\
&= \E_{X \sim \mu}\left[ \lossFunction\left(\phi(X),f_{w(\sigma)}(X)\right)\right] \\
& = \frac{1}{2n}\sum_{s \in [2n]}\lossFunction\left(\phi(x^{(s)}),f_{w(\sigma)}(x^{(s)})\right)\\
& = \frac{1}{2n}\sum_{s \in [2n]} \min\left\lbrace \frac{1}{2^5}\cdot \left( \selfLipschitzConstant \cdot \left\|\phi(x^{(s)})- \left( \left\langle w(\sigma)^j,e({\omega(x^{(s)})}) \right\rangle  \right)_{j \in [\numClasses]}\right\|_{\infty}\right)^{\frac{1}{1-\selfLipschitzExponent}}, 1 \right\rbrace\\
& \geq \frac{1}{2n}\sum_{s \in [2n]} \min\left\lbrace \frac{1}{2^5}\cdot \left( \selfLipschitzConstant \cdot \left|\pi_1(\phi(x^{(s)}))-  \left\langle w(\sigma)^1,e(s) \right\rangle  \right |\right)^{\frac{1}{1-\selfLipschitzExponent}}, 1 \right\rbrace\\
& = \frac{1}{2n}\sum_{r \in [2n]} \min\left\lbrace \frac{1}{2^5}\cdot \left( \selfLipschitzConstant \cdot \left|\pi_1(\phi(x^{(s)}))-   w(\sigma)^1_s  \right |\right)^{\frac{1}{1-\selfLipschitzExponent}}, 1 \right\rbrace\\
& \geq \frac{1}{2n}\sum_{s \in [2n]\backslash A} \min\left\lbrace \frac{1}{2^5}\cdot \left( \selfLipschitzConstant \cdot \left|\pi_1(\phi(x^{(s)}))- \Delta \cdot \sigma_s  \right |\right)^{\frac{1}{1-\selfLipschitzExponent}}, 1 \right\rbrace.
\end{align*}
Observe also that for any $\hat{y} \in \R$, we have $\min_{\sigma_r \in \{-1,+1\}}\left\lbrace  \left|\hat{y}- \Delta \cdot \sigma_r  \right|\right\rbrace \geq \Delta$, by considering the cases $\hat{y}\geq 0$ and $\hat{y}<0$. Moreover, since $\Delta = \sqrt{\kappa/(2n)}$ and $\kappa \leq n \cdot \selfLipschitzConstant^{-2}$ we also have $\selfLipschitzConstant \cdot \Delta \leq 1$. Thus, for all $\hat{y} \in \R$,
\begin{align*}
\min_{\sigma_r \in \{-1,+1\}}\left\lbrace \min\left\lbrace \frac{1}{2^5}\cdot \left( \selfLipschitzConstant \cdot \left|\hat{y}- \Delta \cdot \sigma_r  \right |\right)^{\frac{1}{1-\selfLipschitzExponent}}, 1 \right\rbrace\right\rbrace  \geq \frac{1}{2^5}\cdot (\selfLipschitzConstant \cdot \Delta)^{ \frac{1}{1-\selfLipschitzExponent}}.
\end{align*}
Putting these observations together we have,
\begin{align*}
\E&_{\{\sigma_r\}_{r \in [2n]\backslash A}}\left[\risk\left(\phi,\probDistribution(\sigma)\right)\right]\\ &= \frac{1}{2^{2n-|A|}}\sum_{\{\sigma_r\}_{r \in [2n]\backslash A} \in \{-1,+1\}^{ [2n]\backslash A}}\risk\left(\phi,\probDistribution(\sigma)\right)\\
&\geq  \frac{1}{2^{2n-|A|} \cdot 2n}\sum_{\{\sigma_r\}_{r \in [2n]\backslash A} \in \{-1,+1\}^{ [2n]\backslash A}}\sum_{s \in [2n]\backslash A} \min\left\lbrace \frac{1}{2^5}\cdot \left( \selfLipschitzConstant \cdot \left|\pi_1(\phi(x^{(s)}))- \Delta \cdot \sigma_s  \right |\right)^{\frac{1}{1-\selfLipschitzExponent}}, 1 \right\rbrace\\
&\geq  \frac{1}{2 \cdot 2n} \sum_{s \in [2n]\backslash A}\sum_{\sigma_s \in \{-1,+1\}} \min\left\lbrace \frac{1}{2^5}\cdot \left( \selfLipschitzConstant \cdot \left|\pi_1(\phi(x^{(s)}))- \Delta \cdot \sigma_s  \right |\right)^{\frac{1}{1-\selfLipschitzExponent}}, 1 \right\rbrace\\
&\geq  \frac{1}{2 \cdot 2n} \sum_{s \in [2n]\backslash A} \frac{1}{2^5}\cdot (\selfLipschitzConstant \cdot \Delta)^{ \frac{1}{1-\selfLipschitzExponent}}  = \frac{2n - |A|}{2^6 \cdot 2n}\cdot (\selfLipschitzConstant \cdot \Delta)^{ \frac{1}{1-\selfLipschitzExponent}}.
\end{align*}
\end{proof}

For each $\xSequenceSizeN = \{x_i\}_{ i\in [n]}\in \X^n$ we let $A(\xSequenceSizeN) =\{\omega(x_i)\}_{i \in [n]} \subseteq [2n]$. We have the following independence property.
\begin{lemma}\label{independenceOfManyFlipsFromSampleProofOfLowerBoundLemma} Given any $\xSequenceSizeN = \{x_i\}_{ i\in [n]}\in \X^n$, $\sample = \left\lbrace \left(x_i,f_{w(\sigma)}(x_i)\right)\right\rbrace_{i \in [n]}$ does not depend upon $\{\sigma_r\}_{r \in [2n]\backslash A(x)}$.
\end{lemma}
\begin{proof} For each $i \in [n]$, we have $f_{w(\sigma)}(x_i) = \left( \left\langle w(\sigma)^j,e(\omega(x_i))\right\rangle\right)_{j \in [\numClasses]}= (\Delta \cdot \sigma_{\omega(x_i)} ,0,\cdots,0)$. Hence,  $\sample = \left\lbrace \left(x_i,f_{w(\sigma)}(x_i)\right)\right\rbrace_{i \in [n]}$ does not depend upon $\{\sigma_r\}_{r \in [2n]\backslash A(x)}$. 
\end{proof}
This leads to the following expectation lower bound.
\begin{lemma}\label{centralInExpectationLBLemma} Suppose we have algorithm $\hat{\phi}$, which takes as input a sample $\sample = \{(X_i,Y_i)\}_{ i \in [n]} \in (\X\times \Y)^n$ and outputs a mapping $\hat{\phi}_{\sample} \in \F$. Then we have,
\begin{align*}
\E_{\sigma}\left[ \E_{\sample \sim \probDistribution(\sigma)^n}\left(  \risk\left(\hat{\phi}_{\sample},\probDistribution(\sigma)\right)\right)\right] \geq  \frac{1}{2^7}\cdot (\selfLipschitzConstant \cdot \Delta)^{ \frac{1}{1-\selfLipschitzExponent}}.      
\end{align*}
\end{lemma}
\begin{proof} By Lemma \ref{independenceOfManyFlipsFromSampleProofOfLowerBoundLemma} a data set of the form $ \left\lbrace \left(x_i,f_{w(\sigma)}(x_i)\right)\right\rbrace_{i \in [n]}$ for some $\xSequenceSizeN = \{x_i\}_{ i\in [n]}\in \X^n$ and $\sigma = \{\sigma_r\}_{r \in [2n]}$ is determined solely by $\xSequenceSizeN$ and $\{\sigma_r\}_{r \in A(x)}$, so we write $\sample\left( \xSequenceSizeN, \{\sigma_r\}_{r \in A(x)} \right) = \left\lbrace \left(x_i,f_{w(\sigma)}(x_i)\right)\right\rbrace_{i \in [n]}$. Recall that  for each $\sigma \in \Sigma$ we let $P(\sigma)$ has marginal distribution $\mu$ and for each $x \in \X$, the conditional distribution of $Y$ given $X$ is $P(\sigma)_{Y|x}$ is concentrated on the single point $f_{w(\sigma)}(x)$. Thus, by Lemma \ref{difficultLemmaInRealizableMinimaxLB} we have
\begin{align*}
\E_{\sigma}\left[ \E_{\sample \sim \probDistribution(\sigma)^n}\left[  \risk\left(\hat{\phi}_{\sample},\probDistribution(\sigma)\right)\right]\right] & =\E_{\sigma}\left[ \E_{\bm{X}=\{X_i\}_{i \in [n]} \sim \mu^n}\left[  \risk\left(\hat{\phi}_{\left\lbrace \left(X_i,f_{w(\sigma)}(X_i)\right)\right\rbrace_{i \in [n]}},\probDistribution(\sigma)\right)\right]\right]     \\
 &=\E_{\bm{X} \sim \mu^n}\left[ \E_{\sigma}\left[  \risk\left(\hat{\phi}_{\sample\left( \bm{X}, \{\sigma_r\}_{r \in A(\bm{X})} \right) },\probDistribution(\sigma)\right)\right]\right] \\
&=\E_{\bm{X} \sim \mu^n}\left[ \E_{\{\sigma_r\}_{r \in A(\bm{X})}}\left[ \E_{\{\sigma_r\}_{r \in [2n]\backslash A(\bm{X})}}\left[ \risk\left(\hat{\phi}_{\sample\left( \bm{X}, \{\sigma_r\}_{r \in A(\bm{X})} \right) },\probDistribution(\sigma)\right)\right]\right]\right] \\
&\geq\E_{\bm{X} \sim \mu^n}\left[ \E_{\{\sigma_r\}_{r \in A(\bm{X})}}\left[\frac{2n - |A(\bm{X})|}{2^7\cdot n}\cdot (\selfLipschitzConstant \cdot \Delta)^{ \frac{1}{1-\selfLipschitzExponent}}\right]\right] \geq 2^{-7} \cdot (\selfLipschitzConstant \cdot \Delta)^{ \frac{1}{1-\selfLipschitzExponent}}.
\end{align*}

\end{proof}

We can now complete the proof of Proposition \ref{minimaxOptimalityRelizableProblemsHilbertSpaceLowerBoundProp}.
\begin{proof}[Proof of Proposition \ref{minimaxOptimalityRelizableProblemsHilbertSpaceLowerBoundProp}] Recall that $\Delta = \sqrt{\kappa/(2n)}$, so by Lemma \ref{centralInExpectationLBLemma}  there exists at least one $\sigma \in \Sigma$ for which 
\begin{align*}
\E_{\sample}\left[   \risk\left(\hat{\phi}_{\sample},\probDistribution(\sigma)\right)\right] \geq  2^{-8}\cdot \left(\selfLipschitzConstant \cdot \sqrt{ \frac{\kappa}{n}}\right)^{\frac{1}{1-\selfLipschitzExponent}}.
\end{align*}
Moreover, by Lemma \ref{realizableProblemsLowerBoundLemma} For all $\sigma \in \Sigma$, the probability distribution $\probDistribution(\sigma)$ is a $(\lossFunction,\F)$-realizable problem. This completes the proof of the lower bound.
\end{proof}

\subsection{Proof of Theorem \ref{weCannotExtendTheRangeOfTheSelfBoundingLipschitzExponentThm}}

In this section we take $\X = \{(x_r)_{r \in \N}:\sum_{r \in \N} x_r^2\leq 1\}$, $\Y = \{-1,+1\}$ and $\actionSpace=\R$, and investigate the bounded exponential loss $\lossFunction_{\text{exp}}(u,y)= \min\{1,\exp(-u\cdot y)\}$.

\begin{lemma}\label{boundedExpLossIsVerySelfBoundingLipschitz} Given any $\selfLipschitzExponent \in [0,1]$ and $\selfLipschitzConstant=1$, the bounded exponential loss $\lossFunction_{\text{exp}}$ is $(\selfLipschitzConstant,\selfLipschitzExponent)$-self-bounding Lipschitz.
\end{lemma}
\begin{proof} Fix $\selfLipschitzExponent \in [0,1]$. We must show that for all $u$,$v \in \R$ and $y \in \{-1,+1\}$,
\begin{align}\label{firstClaimOfOfBoundedExpLossIsVerySelfBoundingLipschitz}
|\lossFunction_{\text{exp}}(u,y)-\lossFunction_{\text{exp}}(v,y)| \leq \max\{ \lossFunction_{\text{exp}}(u,y),\lossFunction_{\text{exp}}(v,y)\}^{\selfLipschitzExponent} \cdot |u-v|.
\end{align}
It suffices to prove the claim for the case $y = +1$, since $\lossFunction_{\text{exp}}(u,-1)=\lossFunction_{\text{exp}}(-u,+1)$, so the claim for $y=-1$ will follow. Moreover, without loss of generality we may assume that $\lossFunction_{\text{exp}}(u,y)\leq \lossFunction_{\text{exp}}(v,y)$ which entails $v \leq u$ since $u \mapsto \lossFunction_{\text{exp}}(u,1)= \min\{1,e^{-u}\}$ is non-increasing. There are three cases. Firstly, if $v \leq u \leq 0$ then $\lossFunction_{\text{exp}}(v,y)=\lossFunction_{\text{exp}}(u,y)$ so the claim (\ref{firstClaimOfOfBoundedExpLossIsVerySelfBoundingLipschitz}) holds trivially. Secondly, we consider the case $0 \leq v \leq u$. By the mean value theorem there exists some $w \in [v,u]$ so that 
\begin{align}\label{secondClaimOfOfBoundedExpLossIsVerySelfBoundingLipschitz}
|\lossFunction_{\text{exp}}(u,y)-\lossFunction_{\text{exp}}(v,y)|&= e^{-v}-e^{-u}= e^{-w}\cdot |u-v| \\ &\leq (e^{-v})^{\selfLipschitzExponent}\cdot |u-v| \\ &= \max\{ \lossFunction_{\text{exp}}(u,y),\lossFunction_{\text{exp}}(v,y)\}^{\selfLipschitzExponent} \cdot |u-v|.
\end{align}
This proves the claim (\ref{firstClaimOfOfBoundedExpLossIsVerySelfBoundingLipschitz}) in the second case where $0 \leq v \leq u$. Finally, we turn to the case where $v \leq 0 \leq u$. Here we apply the second case (\ref{secondClaimOfOfBoundedExpLossIsVerySelfBoundingLipschitz}) to obtain
\begin{align*}
|\lossFunction_{\text{exp}}(u,y)-\lossFunction_{\text{exp}}(v,y)|&=|\lossFunction_{\text{exp}}(u,y)-\lossFunction_{\text{exp}}(0,y)|\\ &\leq  \max\{ \lossFunction_{\text{exp}}(u,y),\lossFunction_{\text{exp}}(0,y)\}^{\selfLipschitzExponent}  \cdot |u|\\ &\leq  \max\{ \lossFunction_{\text{exp}}(u,y),\lossFunction_{\text{exp}}(v,y)\}^{\selfLipschitzExponent}  \cdot |u-v|. 
\end{align*}
This covers all possible cases and completes the proof of the lemma.
\end{proof}
To prove Theorem \ref{weCannotExtendTheRangeOfTheSelfBoundingLipschitzExponentThm}, we shall relate the bounded exponential loss $\lossFunction_{\text{exp}}$ to the standard zero one loss $\lossFunction_{\text{0,1}}:\R \times \{-1,+1\} \rightarrow \{0,1\}$ by $\lossFunction_{\text{0,1}}(u,y) = \one\{u \cdot y \leq 0\}$. We shall utilize the following classical result due to \cite{ehrenfeucht1989general} (see also Theorem 3.6, \cite{mohri2012foundations}).

\begin{theorem}[\cite{ehrenfeucht1989general}]\label{realizableVCLowerBoundThm} Let $\X$ be any measurable space and $\HClass \subseteq \measurableMaps(\X,\{-1,+1\})$ of VC dimension at least $d>1$. Then, given any learning algorithm $\hat{\phi}$ which takes as input a sample $\sample = \{(X_i,Y_i)\}_{i \in [n]}$ and outputs $\hat{\phi}_{\sample} \in \HClass$, there exists a probability distribution $\mu$ on $\X$ and a function $\phi^* \in \HClass$ with the following property. Suppose that $\probDistribution_{\mu, \phi^*}$ is the unique probability distribution on $\X \times \Y$ such that $\mu$ is the marginal distribution over $\X$, and for each $x \in \X$, the conditional distribution of $Y$ given $X$, $\probDistribution_{Y|x}$ is concentrated on the point $\phi^*(x)$. Then given a sample $\sample = \{(X_i,Y_i)\}_{i \in [n]}$ with $(X_i,Y_i) \sim \probDistribution_{\mu, \phi^*}$ i.i.d., the following holds with probability at least $1/100$,
\begin{align*}
\riskArg_{\lossFunction_{\text{0,1}}}\left(  \hat{\phi}_{\sample},\probDistribution_{\mu, \phi^*} \right) \geq \frac{d-1}{32 \cdot n}.
\end{align*}
\end{theorem}
The proof of  Theorem \ref{weCannotExtendTheRangeOfTheSelfBoundingLipschitzExponentThm} consists in showing that if the bound in Theorem \ref{optimisticRademacherBoundForMultioutputPrediction} held for some $\selfLipschitzExponent \in (1/2,1]$ then we could produce an algorithm which contradicts the lower bound in Theorem \ref{realizableVCLowerBoundThm}. We require the following conversion from $\measurableMaps(\X,\R)$ to $\measurableMaps(\X,\{-1,+1\})$. Given $f \in \measurableMaps(\X,\R)$ we let $\phi_f \in \measurableMaps(\X,\{-1,+1\})$ denote the map given by
\begin{align*}
\phi_f(x) = \text{sign}\left(f(x)\right)=\begin{cases}+1 &\text{ if }f(x) \geq 0,\\
-1&\text{ if }f(x)<0.
\end{cases}
\end{align*}

\begin{lemma}\label{zeroOneLossLessThanExpLossLemma} Given any $f \in  \measurableMaps(\X,\R)$ and any probability distribution $\probDistribution$ on $\X \times \Y$ we have $\riskArg_{\lossFunction_{\text{0,1}}}\left(  {\phi}_{f},\probDistribution \right) \leq \riskArg_{\lossFunction_{\exp}}\left(  f,\probDistribution \right)$.
\end{lemma}
\begin{proof} It suffices to show that for any $x \in \X$ and $y \in \{-1,+1\}$ we have $\lossFunction_{\text{0,1}}(\phi_f(x),y) \leq \lossFunction_{\exp}(f(x),y)$. Suppose $\lossFunction_{\text{0,1}}(\phi_f(x),y)=1$, so $\phi_f(x) \cdot y \leq 0$ so $f(x) \cdot y \leq 0$, so $1=\lossFunction_{\exp}(f(x),y)= \lossFunction_{\text{0,1}}(\phi_f(x),y)$. Otherwise, $\lossFunction_{\text{0,1}}(\phi_f(x),y)=0 \leq \lossFunction_{\exp}(f(x),y)$.
\end{proof}

 We take $\X = \ell_2 = \{(x_r)_{r \in \N}: \sum_{r \in \N}x_r^2 <\infty\}$. For each $d \in \N $ we take 
\begin{align*}
\HClass_d = \left\lbrace \text{sign}(\left\langle w, x\right\rangle): w = (w_r)_{r \in \N}\in \ell_2\text{ with }w_r= 0 \text{ for }r>d\right\rbrace.
\end{align*}

\begin{lemma}\label{vcDimAndRadComplexityOfDDimensionalHypotheses} For all $d \in \N$ the function class $\HClass_d$ has VC dimension $d$ and for each $n \in \N$ we have 
\begin{align*}
\rademacherComplexity_n(\HClass_d) \leq \sqrt{\frac{2d \log(en/d)}{n}}.
\end{align*}
\end{lemma}
\begin{proof}
See \cite{mohri2012foundations}, Chapter 3 (Corollary 3.1, Example 3.2 and Corollary 3.3).
\end{proof}

We can now conclude the proof by contradiction.
\begin{proof}[Proof of Theorem \ref{weCannotExtendTheRangeOfTheSelfBoundingLipschitzExponentThm}]
Now given $\hypothesisClassBound\geq 1$ and $d \in \N$ we define a corresponding algorithm  $\hat{\phi}^{\hypothesisClassBound,d}$ as follows. Let $\F_{\hypothesisClassBound,d} = \{\hypothesisClassBound \cdot h : h \in \HClass_d\}$. Given a data sample $\sample = \{(X_i,Y_i)\}_{ i \in [n]} \in (\X\times \Y)^n$, we choose $\hat{f}_{\sample} \in \F_{\hypothesisClassBound,d}$ by applying empirical risk minimization within the class $\F_{\hypothesisClassBound,d}$ with respect to the bounded exponential loss $\lossFunction_{\exp}$.  We then take $\hat{\phi}^{\hypothesisClassBound,d}_{\sample}:= \hypothesisClassBound^{-1} \cdot {\hat{f}_{\sample}} \in \HClass_d$. By Theorem \ref{realizableVCLowerBoundThm} and Lemma \ref{vcDimAndRadComplexityOfDDimensionalHypotheses} there exists a probability distribution $\mu$ on $\X$ and a function $\phi^* \in \HClass_d$ such that given a sample $\sample = \{(X_i,Y_i)\}_{i \in [n]}$ with $(X_i,Y_i) \sim \probDistribution_{\mu, \phi^*}$ i.i.d., the following holds with probability at least $1/100$,
\begin{align*}
\riskArg_{\lossFunction_{\text{0,1}}}\left(  \hat{\phi}^{\hypothesisClassBound,d}_{\sample},\probDistribution_{\mu, \phi^*} \right) \geq \frac{d-1}{32 \cdot n}.
\end{align*}
Note also that by construction $\hat{\phi}^{\hypothesisClassBound,d}_{\sample}(x)= \text{sign}(\hat{f}_{\sample}(x))$ for each $x \in \X$, so by Lemma \ref{zeroOneLossLessThanExpLossLemma} the following holds with probability at least $1/100$ over $\sample$,
\begin{align}\label{lowerBoundOnTrueExpRiskInPfOfweCannotExtendTheRangeOfTheSelfBoundingLipschitzExponentThmIneq}
\riskArg_{\lossFunction_{\exp}}\left(  \hat{f}_{\sample},\probDistribution_{\mu, \phi^*} \right) \geq \riskArg_{\lossFunction_{\text{0,1}}}\left(  \hat{\phi}^{\hypothesisClassBound,d}_{\sample},\probDistribution_{\mu, \phi^*} \right)\geq  \frac{d-1}{32 \cdot n}.
\end{align}
By the construction of $\probDistribution_{\mu, \phi^*}$, we have $Y_i = \phi^*(X_i)$ for all $i \in [n]$, with probability one. Moreover, $\phi^* \in \HClass_d$, so $f^* = \hypothesisClassBound \cdot \phi^* \in \F_{\hypothesisClassBound, d}$. Hence, for each $i \in [n]$, $\lossFunction(f^*(X_i),Y_i) = \min\{1, \exp(-f^*(X_i)\cdot Y_i)\} = \exp(-\hypothesisClassBound)$. Thus, with probability one over the sample $\sample$,
\begin{align}\label{upperBoundOnEmpiricalExpRiskInPfOfweCannotExtendTheRangeOfTheSelfBoundingLipschitzExponentThmIneq}
\empiricalRiskArg_{\lossFunction_{\exp}}(\hat{f}_{\sample},\sample)\leq \empiricalRiskArg_{\lossFunction_{\exp}}(f^*,\sample) \leq \exp(-\hypothesisClassBound),
\end{align}
since $\hat{f}_{\sample}$ is the empirical risk minimizer.

\newcommand{\reductioProofComplexityTerm}{\tilde{\Gamma}_{n,\delta}^{\selfLipschitzExponent}(\F_{\beta,d})}

Now suppose, for the purpose of a reductio ad absurdam, that there exists some $\selfLipschitzExponent \in (1/2,1]$ such that the bound in Theorem \ref{optimisticRademacherBoundForMultioutputPrediction} holds. By Lemma \ref{boundedExpLossIsVerySelfBoundingLipschitz} the loss $\lossFunction_{\text{exp}}$ is $(1,\selfLipschitzExponent)$-self-bounding Lipschitz. Hence, there exists a numerical constant $C_0$ such that for any $n\in \N$ and $\delta \in (0,1)$,
\begin{align}\label{relativeBoundOnExpRiskInPfOfweCannotExtendTheRangeOfTheSelfBoundingLipschitzExponentThmIneq}
\riskArg_{\lossFunction_{\exp}}\left(  \hat{f}_{\sample},\probDistribution_{\mu, \phi^*} \right) \leq \empiricalRiskArg_{\lossFunction_{\exp}}(\hat{f}_{\sample},\sample)+ C_0 \cdot \left( \sqrt{\empiricalRiskArg_{\lossFunction_{\exp}}(\hat{f}_{\sample},\sample) \cdot \reductioProofComplexityTerm} +\reductioProofComplexityTerm\right),
\end{align}
where 
\begin{align*}
\reductioProofComplexityTerm:&= \left(  \log^{3/2} \left(e\hypothesisClassBound n\right)\cdot \rademacherComplexity_{n}(\F_{\hypothesisClassBound,d}) + \frac{1}{\sqrt{n}}\right)^{\frac{1}{1-\selfLipschitzExponent}}+\frac{1}{n}\cdot (\log(1/\delta) +\log (\log n))\\
&\leq \left( \hypothesisClassBound \cdot \log^{2} \left(e\hypothesisClassBound n\right)\cdot \sqrt{\frac{2d }{n}} + \frac{1}{\sqrt{n}}\right)^{\frac{1}{1-\selfLipschitzExponent}}+\frac{1}{n}\cdot (\log(1/\delta) +\log (\log n)).
\end{align*}
For the second inequality follows from Lemma \ref{vcDimAndRadComplexityOfDDimensionalHypotheses} since
\begin{align*}
 \rademacherComplexity_{n}(\F_{\hypothesisClassBound,d}) & =  \hypothesisClassBound \cdot \rademacherComplexity_{n}(\HClass_{d}) \leq \hypothesisClassBound \cdot \sqrt{\frac{2d \log(en/d)}{n}} \leq \hypothesisClassBound \cdot \log^{1/2}(e\hypothesisClassBound n) \cdot \sqrt{\frac{2d}{n}}.
\end{align*}
Now take $\hypothesisClassBound = \log n$, $d =\sqrt{n}$ and $\delta= 1/n$, so by combining with (\ref{upperBoundOnEmpiricalExpRiskInPfOfweCannotExtendTheRangeOfTheSelfBoundingLipschitzExponentThmIneq}) we have
\begin{align*}
\empiricalRiskArg_{\lossFunction_{\exp}}(\hat{f}_{\sample},\sample)&\leq \exp(-\hypothesisClassBound) = \frac{1}{n}  \leq \reductioProofComplexityTerm\\
& \leq  \left( \hypothesisClassBound \cdot \log^{2} \left(e\hypothesisClassBound n\right)\cdot \sqrt{\frac{2d }{n}} + \frac{1}{\sqrt{n}}\right)^{\frac{1}{1-\selfLipschitzExponent}}+\frac{1}{n}\cdot (\log(1/\delta) +\log (\log n))\\
& \leq  \left( \log n \cdot \log^{2} \left(e\log n \cdot n\right)\cdot \sqrt{\frac{2n^{1/2} }{n}} + \frac{1}{\sqrt{n}}\right)^{\frac{1}{1-\selfLipschitzExponent}}+\frac{2 \log n}{n}\\
& \leq  100 \cdot  \log^6 (e\cdot n) \cdot n^{-\frac{1}{4(1-\selfLipschitzExponent)}}.
\end{align*}
Moreover, by (\ref{relativeBoundOnExpRiskInPfOfweCannotExtendTheRangeOfTheSelfBoundingLipschitzExponentThmIneq}) this implies that with probability at least $1/n$ we have
\begin{align*}
\riskArg_{\lossFunction_{\exp}}\left(  \hat{f}_{\sample},\probDistribution_{\mu, \phi^*} \right) \leq (1+2C_0) \cdot \reductioProofComplexityTerm \leq  (1+2C_0) \cdot 100 \cdot  \log^6 (e\cdot n) \cdot n^{-\frac{1}{4(1-\selfLipschitzExponent)}}.
\end{align*}
On the other hand, by (\ref{lowerBoundOnTrueExpRiskInPfOfweCannotExtendTheRangeOfTheSelfBoundingLipschitzExponentThmIneq}) the following holds with probability at least $1/100$ for $n\geq 4$,
\begin{align}
\riskArg_{\lossFunction_{\exp}}\left(  \hat{f}_{\sample},\probDistribution_{\mu, \phi^*} \right) \geq \frac{d-1}{32 \cdot n} = \frac{\sqrt{n}-1}{32 \cdot n} \geq 2^{-6} \cdot n^{-1/2}.
\end{align}
Finally, since $\selfLipschitzExponent>1/2$ we have $1/(4(1-\selfLipschitzExponent))>1/2$, so letting $n\rightarrow \infty$ and combining the previous two inequalities gives a contradiction. This contradiction implies that the bound in Theorem \ref{optimisticRademacherBoundForMultioutputPrediction} cannot remain true for $\selfLipschitzExponent \in (1/2,1]$.
\end{proof}

\section{Application to gradient boosting}\label{applicationToGBAppendix}

\newcommand{\lOneBall}{\Lambda_{\lOneRegConstant}}
\newcommand{\lOneUnitBall}{\Lambda_{1}}
\newcommand{\lOneUnitBallCorners}{\Lambda_{1}^{\text{ex}}}
\newcommand{\setOfDecisionTrees}{\mathcal{T}_{\numLeaves,d}}

In this section we complete the proof of Theorem \ref{resultForEnsemblesOfDecisionTrees}. This follows from Theorem \ref{optimisticRademacherBoundForMultioutputPrediction} via the Lemma \ref{rademacherComplexityOfL1NormConstrainedDecisionTrees}, which we shall prove first. Throughout this section we shall assume that $\X = \R^d$. We first define function classes $\overline{\HClass}_{\numLeaves,\lOneRegConstant} \subseteq \measurableMaps(\X,\R^{\numClasses})$ and  ${\HClass}_{\numLeaves,\lOneRegConstant}\subseteq \measurableMaps(\X,[-1,1]^{\numClasses})$ as follows. Firstly, given $\lOneRegConstant>0$ we let $\lOneBall:= \{(a_j)_{j \in [\numClasses]}:\sum_{j \in [\numClasses]}|a_j|\leq \lOneRegConstant\}$. In addition, we let $\setOfDecisionTrees$ be the set of decision trees $t: \R^{d} \rightarrow [\numLeaves]$ with $\numLeaves$ leaves, where each internal node performs a binary split along a single feature. We let $\HClass_{\numLeaves,\lOneRegConstant}$ consists of all functions of the form $h(x) = ( w_{t(x), j})_{j \in [\numClasses]}$, where $t \in \setOfDecisionTrees$, and $\bm{w}=(w_{l,j})_{(l,j) \in [\numLeaves]\times [\numClasses]} \in \left(\lOneBall\right)^{\numLeaves}$, i.e. for each $l \in [\numLeaves]$, we have $(w_{lj})_{j \in [\numClasses]} \in \lOneBall$. We also have ${\HClass}_{\numLeaves,\lOneRegConstant}=\overline{\HClass}_{\numLeaves,\lOneRegConstant}\cap \measurableMaps(\X,[-1,1]^{\numClasses})$, which is equivalent to the definition given in Section \ref{applicationToGBEnsemblesSec}. For the purpose of the proof it is useful to focus on the function classes $\overline{\HClass}_{\numLeaves,\lOneRegConstant}$ for which the output magnitudes are not restricted. This will be necessary for a re-weighting trick in Lemma \ref{convexityOfRademacherComplexity}. We now prove Lemma \ref{rademacherComplexityOfL1NormConstrainedDecisionTrees} and Theorem \ref{resultForEnsemblesOfDecisionTrees}. We begin with the following lemma.

\begin{lemma}\label{empiricalRadEmpacherBoundLOneRegDecisionTreesLemma} For all $m \in \N$ and $\bm{z} \in (\X\times [\numClasses])^m$ we have,
\begin{align*}
\hat{\rademacherComplexity}_{\bm{z}}\left(\Pi \circ \overline{\HClass}_{\numLeaves,\lOneRegConstant} \right) & \leq 2\lOneRegConstant \cdot \sqrt{\frac{\numLeaves\cdot \log (2 \cdot \max\{d\cdot m , \numClasses\})}{m}}.
\end{align*}
\end{lemma}
We begin by counting the number of possible partitions that can be made by a decision tree in $\setOfDecisionTrees$ on a given sequence of points. Given a sequence $\bm{x} = (x_i)_{i \in [m]}\in \X^m$ we let $\setOfDecisionTrees(\bm{x}):= \left\lbrace (t(x_i))_{i \in [m]}:t \in \setOfDecisionTrees\right\rbrace \subseteq [\numLeaves]^m$.

\begin{lemma}\label{countingNumPartitionsLemma} For all $m \in \N$ and $\bm{x} \in \X^m$, $|\setOfDecisionTrees(\bm{x})| \leq (d \cdot (m+1))^{\numLeaves-1}$.
\end{lemma}

\begin{proof} First note that for decision trees $t \in \setOfDecisionTrees$ with $\numLeaves$ leaves which makes binary splits, there are at most $\numLeaves-1$ internal splits. By allowing for \emph{trivial splits} (where all points go along a single branch) we may assume that there exactly $\numLeaves-1$ splits. Each split is a along one of $d$-dimensions and there at most $m+1$ possible ways of performing a binary split of $\{x_i \}_{ \in [m]}$ along a single feature. Putting these facts together proves the lemma.
\end{proof}
We shall utilize Massart's lemma.
\begin{theorem}[\cite{massart2000some}]\label{massartsLemma} Given a bounded set $A \subseteq \R^m$ we have,
\begin{align*}
\E_{\bm{\sigma}}\left( \sup_{(a_i)_{i \in [m]}\in A} \frac{1}{m}\sum_{i \in m} \sigma_i \cdot a_i\right)  \leq \sup_{\bm{a} \in A}\|\bm{a}\|_2 \cdot \frac{ \sqrt{2 \log |A| }}{m},
\end{align*}
where $\sigma_i \in \{-1,+1\}$ are independent Rademacher random variables.
\end{theorem}
\begin{proof} See Theorem 3.3 from \cite{mohri2012foundations}.
\end{proof}
We complete the proof of Lemma \ref{empiricalRadEmpacherBoundLOneRegDecisionTreesLemma} as follows.
\begin{proof}[Proof of Lemma \ref{empiricalRadEmpacherBoundLOneRegDecisionTreesLemma}] 
Let $\{e(j)\}_{j \in [\numClasses]} \subseteq \R^{\numClasses}$ be the canonical orthonormal basis. Let $\lOneUnitBallCorners \subseteq \lOneUnitBall$ denote the subset of extreme points in $\lOneUnitBall$, so $\lOneUnitBallCorners = \{u \cdot e(j): u \in \{-1,+1\}\text{ and }j \in [\numClasses]\}$. Note that $|\lOneUnitBallCorners| = 2\numClasses$. Now fix $\bm{z} = (z_i)_{i \in [m]} \in (\X \times [\numClasses])^m$ with each $z_i = (x_i,j_i)$ and let $\bm{x}=(x_i)_{i \in [m]} \in \X^m$. For each $\bm{\sigma} = (\sigma_i)_{i \in [m]} \in \{-1,+1\}^m$ we have,
\begin{align*}
\sup_{g \in \Pi \circ \overline{\HClass}_{\numLeaves,\lOneRegConstant}} \left\lbrace \frac{1}{m}\sum_{i \in [m]}\sigma_i \cdot g(z_i)\right\rbrace&= \sup_{h \in  \overline{\HClass}_{\numLeaves,\lOneRegConstant}} \left\lbrace\frac{1}{m}\sum_{i \in [m]}\sigma_i \cdot (\Pi\circ h)(x_i,j_i)\right\rbrace\\&=
\sup_{h \in  \overline{\HClass}_{\numLeaves,\lOneRegConstant}} \left\lbrace\frac{1}{m}\sum_{i \in [m]}\sigma_i \cdot \pi_{j_i}(h(x_i))\right\rbrace\\
&=\sup_{t \in  \setOfDecisionTrees} \left\lbrace \sup_{ w \in \left(\lOneBall\right)^{\numLeaves}} \left\lbrace \frac{1}{m}\sum_{i \in [m]}\sigma_i \cdot w_{t(x_i),j_i}\right\rbrace\right\rbrace\\
&=\sup_{(l_i)_{i \in [m]} \in  \setOfDecisionTrees(\bm{x})} \left\lbrace \sup_{ w \in \left(\lOneBall\right)^{\numLeaves}} \left\lbrace \frac{1}{m}\sum_{i \in [m]}\sigma_i \cdot w_{l_i,j_i}\right\rbrace\right\rbrace\\
&= \sup_{(l_i)_{i \in [m]} \in  \setOfDecisionTrees(\bm{x})} \left\lbrace \sup_{ w \in \left(\lOneBall\right)^{\numLeaves}} \left\lbrace \frac{1}{m} \sum_{r \in [\numLeaves]} \sum_{s \in [\numClasses]} \sum_{i: l_i = r \text{ \& }j_i = s}\sigma_i \cdot w_{r,s}\right\rbrace\right\rbrace\\
&= \sup_{(l_i)_{i \in [m]} \in  \setOfDecisionTrees(\bm{x})} \left\lbrace \frac{1}{m} \sum_{r \in [\numLeaves]} \sup_{ (w_{r,s})_{s \in [\numClasses]} \in {\lOneBall}} \left\lbrace  \sum_{s \in [\numClasses]} w_{r,s}\left( \sum_{i: l_i = r \text{ \& }j_i = s}\sigma_i\right) \right\rbrace\right\rbrace\\
&\leq \sup_{(l_i)_{i \in [m]} \in  \setOfDecisionTrees(\bm{x})} \left\lbrace \frac{1}{m} \sum_{r \in [\numLeaves]} \sup_{ (w_{r,s})_{s \in [\numClasses]} \in {\lOneBall}} \left\lbrace  \left(\sum_{s \in [\numClasses]} |w_{r,s}| \right) \cdot\max_{s \in [\numClasses]}\left|\sum_{i: l_i = r \text{ \& }j_i = s}\sigma_i\right| \right\rbrace\right\rbrace\\
&\leq \lOneRegConstant \cdot \sup_{(l_i)_{i \in [m]} \in  \setOfDecisionTrees(\bm{x})} \left\lbrace \frac{1}{m} \sum_{r \in [\numLeaves]} \max_{s \in [\numClasses]}\left|\sum_{i: l_i = r \text{ \& }j_i = s}\sigma_i\right| \right\rbrace\\
&= \lOneRegConstant \cdot  \sup_{(l_i)_{i \in [m]} \in  \setOfDecisionTrees(\bm{x})} \left\lbrace \frac{1}{m} \sum_{r \in [\numLeaves]} \sup_{ (u_{r,s})_{s \in [\numClasses]} \in {\lOneUnitBallCorners}} \left\lbrace  \sum_{s \in [\numClasses]} u_{r,s}\left( \sum_{i: l_i = r \text{ \& }j_i = s}\sigma_i\right) \right\rbrace\right\rbrace \\
&= \lOneRegConstant \cdot  \sup_{(l_i)_{i \in [m]} \in  \setOfDecisionTrees(\bm{x})} \left\lbrace  \sup_{ (u_{r,s}) \in \left(\lOneUnitBallCorners\right)^{\numLeaves}} \left\lbrace \frac{1}{m}  \sum_{i \in [m]}\sigma_i\cdot u_{l_i,j_i}\right\rbrace\right\rbrace,
\end{align*}
where the first inequality follows from H\"{o}lder's inequality. By taking expectations over Rademacher random variables $\bm{\sigma}=(\sigma_i)_{i \in [m]}$, and applying Massart's inequality (Theorem \ref{massartsLemma}) followed by Lemma \ref{countingNumPartitionsLemma} we deduce,
\begin{align*}
\hat{\rademacherComplexity}_{\bm{z}}\left(\Pi \circ \overline{\HClass}_{\numLeaves,\lOneRegConstant}\right)&=\E_{\bm{\sigma}}\left(\sup_{g \in \Pi \circ \overline{\HClass}_{\numLeaves,\lOneRegConstant}} \left\lbrace \frac{1}{m}\sum_{i \in [m]}\sigma_i \cdot g(z_i)\right\rbrace\right)\\
& \leq \lOneRegConstant \cdot  \E_{\bm{\sigma}}\left(\sup_{(l_i)_{i \in [m]} \in  \setOfDecisionTrees(\bm{x})} \left\lbrace  \sup_{ (u_{r,s}) \in \left(\lOneUnitBallCorners\right)^{\numLeaves}} \left\lbrace \frac{1}{m}  \sum_{i \in [m]}\sigma_i\cdot u_{l_i,j_i}\right\rbrace\right\rbrace\right)\\
& \leq \lOneRegConstant \cdot \sup_{(l_i)_{i \in [m]} \in  \setOfDecisionTrees(\bm{x}),\hspace{2mm} (u_{r,s}) \in \left(\lOneUnitBallCorners\right)^{\numLeaves}} \left\lbrace \sqrt{\sum_{i \in [m]} u_{l_i,j_i}^2} \right\rbrace  \cdot \frac{ \sqrt{2 \log \left(\left|\setOfDecisionTrees(\bm{x}) \right| \cdot \left|\lOneUnitBallCorners\right|^{\numLeaves} \right) }}{m}\\
& \leq \lOneRegConstant \cdot \sqrt{\frac{ 2 \left( (\numLeaves-1) \log (d \cdot (m+1))+\numLeaves \cdot \log(2 \numClasses) \right) }{m}}  \leq 2\lOneRegConstant \cdot \sqrt{\frac{\numLeaves\cdot \log (2 \cdot \max\{d\cdot m , \numClasses\})}{m}}.
\end{align*}
\end{proof}

\begin{proof}[Proof of lemma \ref{rademacherComplexityOfL1NormConstrainedDecisionTrees}] Follows immediately from Lemma \ref{empiricalRadEmpacherBoundLOneRegDecisionTreesLemma} by letting $m = n \cdot \numClasses$ and taking a supremum over $\bm{z} \in (\X\times [\numClasses])^{n\cdot \numClasses}$.
\end{proof}
We can now deduce Theorem \ref{resultForEnsemblesOfDecisionTrees} from Theorem \ref{optimisticRademacherBoundForMultioutputPrediction} with the help of a re-weighting argument along with the following standard result. 
\begin{lemma}\label{convexityOfRademacherComplexity} Given a measurable space $\Z$, along with a function class $\G \subseteq \measurableMaps(\Z,\R)$ and a sequence $\bm{z} \in \Z^m$ we have $\hat{\rademacherComplexity}_{\bm{z}}\left(\conv(\G)\right) \leq \hat{\rademacherComplexity}_{\bm{z}}\left(\G\right)$, where $\conv(\G)=\{ \sum_{t \in [T]}\gamma_t\cdot g_t: g_t \in \G,\hspace{2mm}\gamma_t \geq 0\text{ and }\sum_{t \in [T]}\gamma_t \leq 1 \}$.
\end{lemma}

\begin{proof}[Proof of Theorem \ref{resultForEnsemblesOfDecisionTrees}] Take $\zeta>0$ and let
\begin{align*}
\F:= \left\lbrace f= \sum_{t \in [T]}\alpha_t \cdot h_t:\hspace{2mm}h_t \in \HClass_{\numLeaves,\lOneRegConstant_t}, \hspace{2mm}\alpha_t \geq 0,\hspace{2mm}\sum_{t \in [T]}\alpha_t\cdot \lOneRegConstant_t \leq \zeta \text{ and }\sum_{t \in [T]}\alpha_t \leq \hypothesisClassBound\right\rbrace.
\end{align*}
Observe that $\F \subseteq \text{conv}\left(\HClass_{\numLeaves,\zeta}\right)$. Indeed, given  $f= \sum_{t \in [T]}\alpha_t \cdot h_t$ with $h_t \in \HClass_{\numLeaves,\lOneRegConstant_t}$ and  $\sum_{t \in [T]}\alpha_t\cdot \lOneRegConstant_t \leq \zeta$, we have can rewrite 
\begin{align*}
f= \sum_{t \in [T]}\left(\frac{\alpha_t \cdot \lOneRegConstant_t}{\zeta}\right) \cdot ({\zeta}\cdot{\lOneRegConstant_t}^{-1}\cdot h_t),
\end{align*}
with $\sum_{t \in [T]}({\alpha_t \cdot \lOneRegConstant_t}\cdot{\zeta}^{-1})\leq 1$ and for each $t \in [T]$, we have ${\zeta}\cdot{\lOneRegConstant_t}^{-1}\cdot h_t \in \overline{\HClass}_{\numLeaves,\zeta}$. Thus, $\Pi \circ \F \subseteq \Pi \circ \text{conv}\left(\overline{\HClass}_{\numLeaves,\zeta}\right) = \text{conv}\left( \Pi \circ \overline{\HClass}_{\numLeaves,\zeta}\right)$. Hence, by Lemmas \ref{convexityOfRademacherComplexity} and 
$\bm{z} \in (\X\times [\numClasses])^{n\numClasses}$ we have,
\begin{align*}
\hat{\rademacherComplexity}_{\bm{z}}\left(\Pi \circ \F\right) & \leq \hat{\rademacherComplexity}_{\bm{z}}\left(\text{conv}\left( \Pi \circ \overline{\HClass}_{\numLeaves,\zeta}\right) \right) \leq \hat{\rademacherComplexity}_{\bm{z}}\left( \Pi \circ \overline{\HClass}_{\numLeaves,\zeta} \right)\\ &\leq  2\zeta \cdot \sqrt{\frac{\numLeaves\cdot \log (2 \cdot \max\{d\cdot (n \numClasses)  , \numClasses\})}{n \numClasses}}= 2\zeta \cdot \sqrt{\frac{\numLeaves\cdot \log (2 d n \numClasses) }{n \numClasses}}.
\end{align*}
Taking a supremum over all $\bm{z} \in  (\X\times [\numClasses])^{n\numClasses}$ we have $\rademacherComplexity_{n\numClasses}(\Pi \circ \F) \leq 2\zeta \cdot \sqrt{(\numLeaves\cdot \log (2 d n \numClasses))(n \numClasses)^{-1}}$. Note also that $\F \subseteq \measurableMaps(\X,[-\hypothesisClassBound,\hypothesisClassBound]^{\numClasses})$, since for each $f\in \F$ is of the form $f= \sum_{t \in [T]}\alpha_t \cdot h_t$ with $h_t \in \HClass_{\numLeaves,{\lOneRegConstant_t}} \subseteq \measurableMaps(\X,[-1,+1]^{\numClasses})$ and $\sum_{t \in [T]}\alpha_t \leq \beta$. Thus, plugging the bound on  $\rademacherComplexity_{n\numClasses}(\Pi \circ \F) $ into Theorem \ref{optimisticRademacherBoundForMultioutputPrediction} yields the bound in Theorem \ref{resultForEnsemblesOfDecisionTrees}.
\end{proof}

\end{document}